\documentclass{article}

\PassOptionsToPackage{numbers, compress}{natbib}

    \usepackage[preprint]{neurips_2023}

\usepackage[utf8]{inputenc} 
\usepackage[T1]{fontenc}    
\usepackage{hyperref}       
\usepackage{url}            
\usepackage{booktabs}       
\usepackage{amsfonts}       
\usepackage{nicefrac}       
\usepackage{microtype}      
\usepackage{xcolor}         

\usepackage{amsmath}
\usepackage{amssymb}
\usepackage{amsthm}
\usepackage[ruled,vlined]{algorithm2e}
\usepackage{cleveref}
\usepackage{xcolor}         
\usepackage{multirow}
\usepackage{subcaption}
\usepackage{upgreek}
\usepackage{wrapfig}

\usepackage{amsfonts,bm}
\usepackage{color}
\usepackage{pifont}
\usepackage{enumerate}
\usepackage{adjustbox}
\usepackage{ifthen} 
\usepackage{enumitem}
\usepackage{fancybox}
\usepackage[amsstyle]{cases}
\usepackage{empheq}
\usepackage{stmaryrd}
\usepackage{bbm}
\usepackage{physics}
\usepackage{nicefrac}
\usepackage{accents}
\usepackage{xargs}

\usepackage{tikz}
\usepackage{tikzscale}
\usepackage{mathdots}
\usepackage{yhmath}
\usepackage{cancel}
\usepackage{siunitx}
\usepackage{array}
\usepackage{gensymb}
\usepackage{tabularx}
\usepackage{extarrows}
\usetikzlibrary{fadings}
\usetikzlibrary{patterns}
\usetikzlibrary{shadows.blur}
\usetikzlibrary{shapes}

\usetikzlibrary{decorations.pathreplacing,calc,calligraphy}

\theoremstyle{definition}
\newtheorem{theorem}{Theorem}
\newtheorem{lemma}{Lemma}

\newtheorem{corollary}{Corollary}
\newtheorem{remark}{Remark}
\newtheorem{assumption}{Assumption}

\crefname{equation}{}{}
\crefname{algorithm}{Alg.}{Alg.}
\crefname{line}{Line}{Lines}
\crefname{section}{Section}{Sections}
\crefname{assumption}{Assumption}{Assumptions}
\crefname{definition}{Definition}{Definitions}
\crefname{lemma}{Lemma}{Lemmas}
\crefname{question}{Question}{Questions}
\crefname{figure}{Fig.}{Fig.}
\crefname{table}{Table}{Tables}
\crefname{remark}{Remark}{Remarks}
\crefname{theorem}{Theorem}{Theorems}
\crefname{corollary}{Corollary}{Corollaries}
\crefname{appendix}{Appendix}{Appendices}

\renewcommand{\paragraph}[1]{\noindent\textbf{#1}\: }

\usepackage[normalem]{ulem}

\newif\ifedit
\edittrue

\title{Decentralized Hyper-Gradient Computation\\over Time-Varying Directed Networks}

\usepackage{authblk}

\author[12\thanks{\tt{naoyuki.terashita.sk@hitachi.com}}]{Naoyuki Terashita}
\author[2]{Satoshi Hara}

\affil[1]{Hitachi, Ltd.}
\affil[2]{Osaka University}

\begin{document}

\maketitle

\begin{abstract}
This paper addresses the communication issues when estimating hyper-gradients in decentralized federated learning (FL).
Hyper-gradients in decentralized FL quantifies how the performance of globally shared optimal model is influenced by the perturbations in clients' hyper-parameters.
In prior work, clients trace this influence through the communication of Hessian matrices over a static undirected network, resulting in (i) excessive communication costs and (ii) inability to make use of more efficient and robust networks, namely, time-varying directed networks.
To solve these issues, we introduce an alternative optimality condition for FL using an averaging operation on model parameters and gradients. 
We then employ Push-Sum as the averaging operation, which is a consensus optimization technique for time-varying directed networks.
As a result, the hyper-gradient estimator derived from our optimality condition enjoys two desirable properties; (i) it only requires Push-Sum communication of vectors and (ii) it can operate over time-varying directed networks.
We confirm the convergence of our estimator to the true hyper-gradient both theoretically and empirically, and we further demonstrate that it enables two novel applications: decentralized influence estimation and personalization over time-varying networks.
Code is available at \url{https://github.com/hitachi-rd-cv/pdbo-hgp.git}. 
\end{abstract}

\section{Introduction} 
\subsection{Background} \label{sec:intro_bg}
Hyper-gradient has gained attention for addressing various challenges in federated learning (FL)~\citep{McMahan2017}, such as preserving fairness among clients in the face of data heterogeneity~\citep{Lu2022,Li2022}, tuning hyper-parameters with client cooperation~\citep{Yang2022,Gao2022,Chen2022DBO}, and improving the interpretability of FL training~\citep{xue2021toward}.

This paper primarily focuses on hyper-gradient computation in decentralized (or peer-to-peer) FL~\citep{Kairouz2021}[1.1], under practical consideration for communications.
Decentralized FL is known to offer stronger privacy protection~\citep{Cyffers2022}, faster model training~\citep{Lian2017,Marfoq2020}, and robustness against slow clients~\citep{8737602}.
However, these properties of decentralized FL also bring unique challenges in the hyper-gradient estimation.
This is because clients must communicate in a peer-to-peer manner to measure \textit{how the perturbations on hyper-parameters of individual clients alter the overall performance of the shared optimal model},  requiring the careful arrangement of what and how clients should communicate.

Specifically, there are two essential challenges: (i) \emph{communication cost} and (ii) \emph{configuration of communication network}.
We provide a brief overview of these challenges below and in \cref{table:dbo}.

\noindent\textbf{Communication cost}\: 
In centralized FL, the central server can gather all necessary client information for hyper-gradient computation, enabling a communication-efficient algorithm as demonstrated by \citet{Tarzanagh2022}.
However, designing such an efficient algorithm for decentralized FL is more challenging, as clients need to perform the necessary communication and computation without central orchestration.
This challenge results in less efficient algorithms~\citep{Yang2022,Chen2022DBO} as shown in \tablename~\ref{table:dbo}; the current decentralized hyper-gradient computations require large communication costs for exchanging Hessian matrices.

\noindent\textbf{Configuration of communication network}\: There are several types of communication network configurations for decentralized FL. 
One of the most general and efficient configurations is time-varying directed communication networks, which allow any message passing to be unidirectional.
This configuration is known to be resilient to failing clients and deadlocks~\citep{Tsianos2012} at a minimal communication overhead~\citep{Assran2019}. 
However, hyper-gradient computation on such a dynamic network remains unsolved, and previous approaches operate over less efficient configurations as shown in \tablename~\ref{table:dbo}.

\begin{table}[t]
	\caption{Concise comparison of the hyper-gradient for FL.}
	\label{table:dbo}
	\centering
    \small
\begin{tabular}{lcccc}
	\toprule
        & Decentralized? &  Communication Cost & Communication Network \\\midrule 
        \citet{Tarzanagh2022} & No & \textbf{Small} & Centralized \\ \midrule
        \citet{Chen2022DBO} & \textbf{Yes} & Large & Static Undirected \\ 
        \citet{Yang2022} & \textbf{Yes} & Large & Static Undirected \\ 
        \textbf{HGP (Ours)} &  \textbf{Yes} &  \textbf{Small}  & \textbf{Time-Varying Directed} \\ \bottomrule
\end{tabular}
\end{table}

\subsection{Our Contributions}
In this paper, we demonstrate that both problems can be solved simply by introducing an appropriate optimality condition of FL which is utilized to derive the hyper-gradient estimator.
We found that the optimality condition of decentralized FL can be expressed by the averaging operation on model parameters and gradients.
We then select Push-Sum~\citep{benezit2010weighted} as the averaging operation,  which is known as a consensus optimization that runs over time-varying directed networks.
Based on our findings and the specific choice of Push-Sum for average operation, we propose our decentralized algorithm for hyper-gradient computation, named Hyper-Gradient Push (HGP), and provide its theoretical error bound.
Notably, the proposed HGP resolves the aforementioned two problems; (i) it communicates only vectors using Push-Sum, avoiding exchanging Hessian matrices, and (ii) it can operate on time-varying networks, which is more efficient and robust than static undirected networks.

Our numerical experiments confirmed the convergence of HGP towards the true hyper-gradient.
We verified the efficacy of our HGP on two tasks: influential training instance estimation and model personalization.
The experimental results demonstrated that our HGP enabled us, for the first time, to solve these problems over time-varying communication networks.
For personalization, we also verified the superior performance of our HGP against existing methods on centralized and static undirected communication networks.

Our contributions are summarized as follows:
\begin{itemize}[leftmargin=*]
    \item We introduce a new formulation of hyper-gradient for decentralized FL using averaging operation, which can be performed by Push-Sum iterations. This enabled us to design the proposed HGP; it only requires the communication of the model parameter-sized vectors over time-varying directed networks. We also provide a theoretical error bound of our estimator.
    \item We empirically confirmed the convergence of our estimation to the true hyper-gradient through the experiment. We also demonstrated two applications that are newly enabled by our algorithm: influence estimation and personalization over time-varying communication networks.
\end{itemize}

\medskip
\paragraph{Notation}
$\langle \boldsymbol{A} \rangle_{ij}$ denotes the \(i\)-th row and \(j\)-th column block of the matrix \(\boldsymbol{A}\) and \(\langle \boldsymbol{a} \rangle_{i}\) denotes the $i$-th block vector of the vector \(\boldsymbol{a}\).
For a vector function $\boldsymbol{h}: \mathbb{R}^{m} \to \mathbb{R}^{n} $, we denote its total and partial derivatives by $\mathrm{d}_{\boldsymbol{x}}\boldsymbol{h}(\boldsymbol{x}) \in \mathbb{R}^{m\times n}$ and $\partial_{\boldsymbol{x}}\boldsymbol{h}(\boldsymbol{x}) \in \mathbb{R}^{m\times n}$, respectively.
For a real-valued function $h: \mathbb{R}^{m} \times \mathbb{R}^{s} \to \mathbb{R}$, we denote the Jacobian of  $\partial_{\boldsymbol{x}}h(\boldsymbol{x},\boldsymbol{y})$ with respect to $\boldsymbol{x}$ and $\boldsymbol{y}$ by $\partial_{\boldsymbol{x}}^2{h}(\boldsymbol{x},\boldsymbol{y}) \in\mathbb{R}^{m\times m} $ and $\partial_{\boldsymbol{x}\boldsymbol{y}}^2{h}(\boldsymbol{x},\boldsymbol{y}) \in\mathbb{R}^{s\times m} $, respectively.
We also introduce a concatenating notation \([\boldsymbol{z}_i]_{i=1}^{n}=[\boldsymbol{z}_1^\top~\cdots~\boldsymbol{z}_n^\top]{}^{\top} \in \mathbb{R}^{nd}\) for vectors \(\boldsymbol{z}_i \in \mathbb{R}^{d}\).
We denote the largest and smallest singular values of a matrix $\boldsymbol{A}$ by $ \sigma_{\max}(\boldsymbol{A})$ and $ \sigma_{\min}(\boldsymbol{A})$, respectively.

\section{Preliminaries}
This section provides the background of our study.
\cref{sec:pushsum} introduces the model of the time-varying network and Push-Sum algorithm.
\cref{sec:fl} provides the formulation of decentralized FL and its optimality condition, then \cref{sec:hg} presents a typical approach for hyper-gradient estimation in a single client setting.

\subsection{Time-Varying Directed Networks and Push-Sum}  \label{sec:pushsum}
\textbf{Time-varying directed communication networks}, in which any message passing can be unidirectional, have proven to be resilient to failing clients and deadlocks~\citep{Tsianos2012} and they enjoy the minimal communication overhead~\citep{Assran2019}. 
We denote the time-varying directed graph at a time step index $s>0$ by  $\mathcal{G}(s)$ with vertices $\{1,\ldots,n\}$ and edges defined by $\mathcal{E}(s)$.

We suppose that at step \(s\),  any $i$-th client sends messages to its out-neighborhoods $\mathcal{N}_{i}^{\mathrm{out}}(s) \subseteq \{1,\ldots,n\}$ and receives messages from the in-neighborhoods $\mathcal{N}_{i}^{\mathrm{in}}(s)$.
In addition, by standard practice, every $i$-th client is always regarded as its own in-neighbor and out-neighbor, i.e.,  $i \in \mathcal{N}_{i}^{\mathrm{out}}(s)$ and  $i \in \mathcal{N}_{i}^{\mathrm{in}}(s)$ for all $i,s$.
We also introduce an assumption on the connectivity of $\mathcal{G}(s)$ following \citet{Nedic2016}.
Roughly speaking, \cref{ass:connect} requires the time-varying network $\mathcal{G}(s)$ to be repeatedly connected over some sufficiently long time scale  $B > 0$.
\begin{assumption} \label{ass:connect}
The graph with edge set $\bigcup _{s=tB}^{(t+1)B-1}\mathcal{E}(s)$ is strongly-connected for every $t\geq 0$.
\end{assumption}
 
\begin{wrapfigure}{r}{0.38\textwidth}
\vskip -0.2in
\begin{algorithm}[H] 
\caption{Push-Sum} \label{alg:pushsum}
\DontPrintSemicolon
\KwIn{$\boldsymbol{y}_{i}^{(0)}$}
$\boldsymbol{z}_{i}^{(0)} \leftarrow \boldsymbol{y}_{i}^{(0)},~\omega_{i}^{(s)}\leftarrow 1$\;
\For{$s = 1$ \KwTo $S$}{
    $\boldsymbol{z}_{i}^{(s)} \leftarrow \sum_{j\in \mathcal{N}_{i}^{\mathrm{in}}(s)} \frac{\boldsymbol{z}_{j}^{(s-1)}}{\vert\mathcal{N}_{j}^{\mathrm{out}}(s)\vert} $\;
    $\omega_{i}^{(s)} \leftarrow \sum_{j\in \mathcal{N}_{i}^{\mathrm{in}}(s)} \frac{\omega_{j}^{(s-1)}}{\vert\mathcal{N}_{j}^{\mathrm{out}}(s)\vert}$\;
    $\boldsymbol{y}_{i}^{(s)} \leftarrow \frac{\boldsymbol{z}_{i}^{(s)}}{\omega_{i}^{(s)}}$\;
}
\KwOut{$\boldsymbol{y}_{i}^{(S)}$}
\end{algorithm}
\vskip -0.2in
\end{wrapfigure}

\textbf{Push-Sum}~\citep{benezit2010weighted} (\cref{alg:pushsum}) is an algorithm for computing an average of values possessed by each client through communications over time-varying directed networks $\mathcal{G}(s)$ satisfying \cref{ass:connect}. 
When each $i$-th client runs \cref{alg:pushsum} from its initial value vector $\boldsymbol{y}_{i}^{(0)} \in \mathbb{R}^{d}$, it eventually obtains the average of initial values (or consensus) over the clients~\citep{nedic2014distributed}, i.e., $\lim _{S\rightarrow \infty } \boldsymbol{y}_{i}^{(S)} =\frac{1}{n}\sum _{k} \boldsymbol{y}_{k}^{(0)}$.
From this property, we can regard Push-Sum as a linear operator $\Theta$.
Namely, denoting concatenated vectors by $\boldsymbol{y}^{(0)}=[\boldsymbol{y}_{i}^{(0)}]_{i=1}^{n}$, $\bar{\boldsymbol{y}} = [\frac{1}{n}\sum _{k} \boldsymbol{y}_{k}^{(0)}]_{i=1}^{n}$, we have
\begin{equation} \label{eq:lin_op}
\boldsymbol{\Theta}\boldsymbol{y}^{(0)} =\bar{\boldsymbol{y}}, \quad \left\langle\boldsymbol{\Theta}\right\rangle_{ij} = \frac{1}{n} \boldsymbol{I}_{d},~~\forall i,j = 1,\ldots,n,
\end{equation}
where, $\boldsymbol{I}_{d}$ denotes the identity matrix with size $d\times d$.
Finally, we remark on a useful consequence of this section, which takes an important role in our decentralized hyper-gradient estimation:
\begin{remark} \label{remark:pushsum}
    When \cref{ass:connect} is satisfied and every $j$-th client knows $\boldsymbol{y}_{j}^{(0)}$, any $i$-th client can obtain $\langle \boldsymbol{\Theta} \boldsymbol{y}^{(0)}\rangle_{i}$ by communications over time-varying directed networks.
\end{remark}

\subsection{Decentralized Federated Learning}  \label{sec:fl}
The federated learning (FL)~\citep{McMahan2017} consisting of $n$ clients is formulated by
\begin{equation} \label{eq:fl}
\min_{\boldsymbol{x}_{1}, \ldots, \boldsymbol{x}_{n}}\sum _{k=1}^{n}\mathbb{E} [g_{k} (\boldsymbol{x}_{k} ,\boldsymbol{\lambda} _{k} ;\xi _{k} )], \quad  \mathrm{s.t.} \quad \boldsymbol{x}_{i} =\boldsymbol{x}_{j} ,\ \forall i, j, 
\end{equation}
where, $g_{i} :\mathbb{R}^{d_{\boldsymbol{x}}} \times \mathbb{R}^{d_{\boldsymbol{\lambda} }}\rightarrow \mathbb{R}$ is a cost function of the $i$-th client, and $\xi _{i}$ denotes a random variable that represents the instance only accessible by the $i$-th client.
Note that the distribution of $\xi _{i}$ may differ between each client.
In decentralized FL, the objective of each $i$-th client is to find $\boldsymbol{x}_{i}$ that minimizes the total cost while maintaining \textit{consensus constraint}, i.e., $\boldsymbol{x}_{i} =\boldsymbol{x}_{j}, \forall i, j$.
Stochastic gradient push~\citep{Nedic2016,Assran2019} enables us to solve (\ref{eq:fl}) over time-varying directed networks.

\subsection{Hyper-Gradient Computation} \label{sec:hg}
Hyper-gradient is an effective tool for solving bilevel problems, which is the nested problem consisting of inner- and outer-problem~\citep{Domke2012,Maclaurin2015,Pedreska2016}.
Hyper-gradient can also be used for influential training instance estimation, which studies how the removal of a training instance influences the performance of the optimal model~\citep{Koh2017}. 
Below, we explain the definition and computation method of the hyper-gradient in the context of the bilevel problem.

Using differentiable function $f$ and $g$, the bilevel problem is formulated by
\begin{equation} \label{eq:bilevel}
    \underbrace{\min_{\boldsymbol{\lambda} \in \mathbb{R}^b} f(\boldsymbol{x}(\boldsymbol{\lambda}), \boldsymbol{\lambda})}_{\text{outer-problem}} , \quad  \mathrm{s.t.} \quad \underbrace{\boldsymbol{x}(\boldsymbol{\lambda}) = \min_{\boldsymbol{x} \in \mathbb{R}^a} g(\boldsymbol{x}, \boldsymbol{\lambda})}_{\text{inner-problem}} .
\end{equation}
Suppose that the optimal solution of the inner-problem $\boldsymbol{x} (\boldsymbol{\lambda} ) \in \mathbb{R}^a$ is expressed by the stationary point given by a differentiable function $\boldsymbol{\varphi} :\mathbb{R}^{a}\times \mathbb{R}^{b}\rightarrow \mathbb{R}^a$:
\begin{equation} \label{eq:stationary}
\boldsymbol{x} (\boldsymbol{\lambda} )=\boldsymbol{\varphi} \left( \boldsymbol{x} (\boldsymbol{\lambda}), \boldsymbol{\lambda} \right) .
\end{equation}
For example, if $g$ is smooth and strongly convex with respect to $\boldsymbol{x}$, we can use $\boldsymbol{\varphi}(\boldsymbol{x},\boldsymbol{\lambda})=\boldsymbol{x}-\eta \partial_{\boldsymbol{x}}g(\boldsymbol{x},\boldsymbol{\lambda})$ with $\eta > 0$ to express the optimality condition $\partial_{\boldsymbol{x}}g(\boldsymbol{x},\boldsymbol{\lambda})=0$ using (\ref{eq:stationary}).

For the bilevel problem (\ref{eq:bilevel}), we refer $\mathrm{d}_{\boldsymbol{\lambda}} f \left( \boldsymbol{x} (\boldsymbol{\lambda}), \boldsymbol{\lambda} \right)$ as hyper-gradient\footnote{In the remainder of the paper, we omit the arguments $(\boldsymbol{x}(\boldsymbol{\lambda}), \boldsymbol{\lambda})$ when it is clear from the context, e.g., $\boldsymbol{\varphi} = \boldsymbol{\varphi}(\boldsymbol{x}(\boldsymbol{\lambda}), \boldsymbol{\lambda})$.}.
One of the most common approach for computing $\mathrm{d}_{\boldsymbol{\lambda}} f$ is the fixed-point method~\citep{Pedreska2016,Lorraine2020}. 
When $\partial_{\boldsymbol{x}}\boldsymbol{\varphi}$ is positive-semidefinite and has its eigenvalues smaller than one, by the derivative of \cref{eq:stationary} and Neumann approximation of inverse, we obtain $\mathrm{d}_{\boldsymbol{\lambda}}\boldsymbol{x} (\boldsymbol{\lambda} ) = \partial _{\boldsymbol{\lambda}} \boldsymbol{\varphi}\left(\boldsymbol{I} - \partial _{\boldsymbol{x}} \boldsymbol{\varphi} \right)^{-1} =\partial _{\boldsymbol{\lambda}} \boldsymbol{\varphi}\sum _{m=0}^{\infty } (\partial _{\boldsymbol{x}} \boldsymbol{\varphi}) ^{m}$, leading to
\begin{equation} \label{eq:fixed_point}
\mathrm{d}_{\boldsymbol{\lambda}} f=\partial _{\boldsymbol{\lambda}} \boldsymbol{\varphi}\sum _{m=0}^{\infty } (\partial _{\boldsymbol{x}} \boldsymbol{\varphi} )^{m} \partial _{\boldsymbol{x}} f + \partial  _{\boldsymbol{\lambda}} f .
\end{equation}
Fixed-point method also provides an efficient algorithm to compute \cref{eq:fixed_point}:
\begin{subequations}
 \begin{align} 
\text{(initialization)} \quad & \boldsymbol{v}^{(0)} =\partial _{\boldsymbol{\lambda}}f , \quad && \boldsymbol{u}^{(0)} =\partial _{\boldsymbol{x}} f , \label{eq:init} \\
\text{(iteration for $m=1, \ldots, M$)} \quad & \boldsymbol{v}^{(m)} =\partial _{\boldsymbol{\lambda}} \boldsymbol{\varphi} \boldsymbol{u}^{(m-1)} +\boldsymbol{v}^{(m-1)}, \quad && \boldsymbol{u}^{(m)} =\partial _{\boldsymbol{x}} \boldsymbol{\varphi}  \boldsymbol{u}^{(m-1)}  \label{eq:recur},
\end{align}
\end{subequations}
which results in \(\boldsymbol{v}^{(M)} =\partial _{\boldsymbol{\lambda}}  \boldsymbol{\varphi} \sum_{m^{\prime}=0}^{M-1} (\partial _{\boldsymbol{x}} \boldsymbol{\varphi} )^{m'} + \partial _{\boldsymbol{\lambda}} f \approx \mathrm{d}_{\boldsymbol{\lambda}} f\). 
Here, no explicit computation of Jacobians are required in \cref{eq:recur}; $\partial _{\boldsymbol{\lambda}} \boldsymbol{\varphi} \boldsymbol{u}^{(m-1)}$ and $\partial _{\boldsymbol{x}} \boldsymbol{\varphi}  \boldsymbol{u}^{(m-1)}$ can be computed using the Jacobian-vector-product technique.

\section{Estimating Hyper-gradient over Time-varying Directed Networks}
In this section, we first explain the main technical challenge of hyper-gradient computation in distributed FL, namely, large communication costs due to the consensus constraint in the optimality condition of FL~\cref{eq:fl}.
We then introduce our alternative optimality condition using the convergence of Push-Sum. 
By using our optimality condition, we finally propose the decentralized hyper-gradient estimation algorithm HGP that runs with reasonable communication cost over time-varying networks.

\subsection{Main Challenge} \label{sec:challenge}
We consider the stationary point of decentralized FL \cref{eq:fl} and hyper-gradient derived from this stationary point. 
Let $\boldsymbol{\lambda} =\left[\boldsymbol{\lambda} _{i} \right]_{i=1}^{n} \in \mathbb{R}^{nd_{\boldsymbol{\lambda} }}$ and $\boldsymbol{x}=\left[\boldsymbol{x}_{i} \right]_{i=1}^{n} \in \mathbb{R}^{nd_{\boldsymbol{x}}}$ be concatenated inner-parameters and hyper-parameters, respectively.
We also denote the expectation of total inner-cost by $ g  (\boldsymbol{x},\boldsymbol{\lambda})=\sum _{k=1}^{n} \mathbb{E} [g_{k} (\boldsymbol{x}_{k} ,\boldsymbol{\lambda} _{k} ;\xi _{k} )] $ with the following assumption. 
\begin{assumption} \label{ass:convex}
    For every $i=1,\ldots,n$, $g_{i}$ is strongly convex with respect to the first argument.
\end{assumption}
We can then reformulate the optimality condition of \cref{eq:fl} by the stationary point \cref{eq:stationary} with
\begin{equation} \label{eq:map_ori}
\boldsymbol{\varphi} \left( \boldsymbol{x},\boldsymbol{\lambda} \right) =\boldsymbol{x} -\eta \partial _{\boldsymbol{x}}  g \left( \boldsymbol{x} ,\boldsymbol{\lambda}\right),  \quad \mathrm{s.t.} \quad \left\langle  \boldsymbol{x} \right\rangle_{i} = \left\langle  \boldsymbol{x}  \right\rangle_{j} \in \mathbb{R}^{d_{\boldsymbol{x}}}, \forall i,j ,
\end{equation}
where $\eta \in \mathbb{R}^{+}$.
Here, the latter constraint corresponds to the consensus constraint in \cref{eq:fl} and \cref{ass:convex} ensures the existence of $\left(\boldsymbol{I} - \partial _{\boldsymbol{x}} \boldsymbol{\varphi} \right)^{-1}$.

Let $ f(\boldsymbol{x},\boldsymbol{\lambda})=\sum _{k} f_{k} (\boldsymbol{x}_{k} ,\boldsymbol{\lambda} _{k} )$ be an outer-cost in bilevel decentralized FL.
Here, each $i$-th client is interested in the hyper-gradient $ f  (\boldsymbol{x}(\boldsymbol{\lambda}),\boldsymbol{\lambda})$ with respect to its hyper-parameter $\boldsymbol{\lambda}_i$.
The technical challenge is to compute $\mathrm{d}_{\boldsymbol{\lambda}}  f $ in a decentralized manner, especially computing \cref{eq:recur}. 
From the consensus constraint~\cref{eq:fl}, for any $m\geq0$, any block vector of $\boldsymbol{u}^{(m)}$ requires the evaluation of $f_k$ and $g_k$ for all $k=1,\ldots,n$, because of the following blocks in \cref{eq:recur}:
\begin{align}
 \langle \partial_{\boldsymbol{x}}  f   \rangle _{i} &=\eta \sum _{k}\partial _{\boldsymbol{x_k}} f_{k} (\boldsymbol{x}_{k}(\boldsymbol{\lambda}) ,\boldsymbol{\lambda} _{k} ), \quad
  \left\langle \partial _{\boldsymbol{x}} \boldsymbol{\varphi}\right\rangle_{ij}  =\boldsymbol{I}-\eta \sum _{k}\mathbb{E}[\partial _{\boldsymbol{x}_k} ^{2}g_k (\boldsymbol{x}_{k}(\boldsymbol{\lambda}) ,\boldsymbol{\lambda}_k;\xi_k ) ], ~\forall i,j. \label{eq:phi_all}
\end{align}
A naive computation of \cref{eq:phi_all} requires gathering these derivatives from all clients through communications.
The communication of the Hessian $\partial _{\boldsymbol{x}_k} ^{2}g_k (\boldsymbol{x}_{k}(\boldsymbol{\lambda}) ,\boldsymbol{\lambda}_k;\xi_k ) $ is particularly a problem for large models such as deep neural networks.

In the next section, we show that there is an alternative yet equivalent stationary condition that does not explicitly require consensus between any clients.
Based on this alternative condition, we introduce the proposed HGP, a fixed-point iteration without requiring exchanging Hessian matrices which can run even on time-varying directed networks. 

\subsection{Alternative yet Equivalent Stationary Condition}
We first present the alternative stationary condition as follows:
\begin{lemma} \label{lem:proximal}
A stationary condition $\boldsymbol{x} (\boldsymbol{\lambda} ) =\boldsymbol{\varphi}^{\prime} \left( \boldsymbol{x} (\boldsymbol{\lambda}), \boldsymbol{\lambda} \right)$ with a function
\begin{equation}\label{eq:stational_alt}
    \boldsymbol{\varphi}^{\prime } (\boldsymbol{x},\boldsymbol{\lambda})=\boldsymbol{\Theta} \left(\boldsymbol{x} -\eta  \partial _{\boldsymbol{x}} g \left( \boldsymbol{x},\boldsymbol{\lambda} \right)\right),
\end{equation}
holds true if and only if $\boldsymbol{x}(\boldsymbol{\lambda})$ is the solution of \cref{eq:fl}.
\end{lemma}
\cref{lem:proximal} states that each $\boldsymbol{x}_i$ is the optimal solution of FL only when it is identical to their average and when the average of the gradients is zero. 
While both \cref{eq:stational_alt} and \cref{eq:map_ori} characterize the optimality condition \cref{eq:fl} though their stationary condition, our \cref{eq:stational_alt} has the following desirable property:
\begin{remark} \label{remark:update}
\cref{lem:proximal} requires $\boldsymbol{x}_{i} =\boldsymbol{x}_{j}$ only implicitly. Thus, any block of the partial derivative with respect to $\boldsymbol{x}$ can be calculated by a single client.
\end{remark}

\subsection{Stochastic and Decentralized Approximation of Hyper-Gradient}
Finally, we present our decentralized algorithm, named Hyper-Gradient Push (HGP).

Since \cref{ass:convex} ensures $\left( I-\partial _{\boldsymbol{x}} \boldsymbol{\varphi}^{\prime} \right)^{-1}$ exists, we can derive the hyper-gradient similar to \cref{eq:fixed_point}:
\begin{equation*}
\mathrm{d}_{\boldsymbol{\lambda}}  f =-\eta \partial _{\boldsymbol{x}\boldsymbol{\lambda}}^{2} g   \boldsymbol{\Theta} \sum\limits _{m=0}^{\infty }\left(\left( \boldsymbol{I}-\eta \partial _{\boldsymbol{x}}^{2} g   \right) \boldsymbol{\Theta} \right)^{m} \partial _{\boldsymbol{x}} f   +\partial _{\boldsymbol{\lambda}} f ,  
\end{equation*}
where we used  $\boldsymbol{\Theta} ^{\top } =\boldsymbol{\Theta} $ from \cref{eq:lin_op}.
Similar to \cref{eq:recur}, we can obtain the fixed-point iteration of the form
\begin{equation} \label{eq:recur_prox}
\boldsymbol{v}^{(m)} =-\eta \partial _{\boldsymbol{x}\boldsymbol{\lambda}}^{2} g   \boldsymbol{\Theta} \boldsymbol{u}^{(m-1)} +\boldsymbol{v}^{(m-1)},  \quad \boldsymbol{u}^{(m)} =\left( \boldsymbol{I}-\eta \partial _{\boldsymbol{x}}^{2} g   \right) \boldsymbol{\Theta} \boldsymbol{u}^{(m-1)}. 
\end{equation}

\begin{algorithm}[t]  
\caption{Hyper-Gradient Push~(HGP)} \label{alg:hgp}
\DontPrintSemicolon
\KwIn{$\boldsymbol{x}_i(\boldsymbol{\lambda})$, $\boldsymbol{\lambda}_i$, $M$, $\eta$}
$\boldsymbol{u}_i^{(0)} \leftarrow \partial _{\boldsymbol{x}} f_i(\boldsymbol{x}_i(\boldsymbol{\lambda}), \boldsymbol{\lambda}_i)$, $\boldsymbol{v}_i^{(0)} \leftarrow \partial _{\boldsymbol{\lambda}} f_i(\boldsymbol{x}_i(\boldsymbol{\lambda}), \boldsymbol{\lambda}_i)$\;
\For{$m = 1$ \KwTo $M$}{
    Run Push-Sum (\cref{alg:pushsum}) with $\boldsymbol{u}_i^{(m-1)}$ and obtain the estimated average $\bar{\boldsymbol{u}}_i^{(m-1)}$\;
    Sample instances $\hat{\xi}_i^{(m)}$ and $\hat{\hat{\xi}}_i^{(m)}$ independently \;
    $\boldsymbol{v}_i^{(m)} \leftarrow \boldsymbol{v}_i^{(m-1)} - \eta\partial _{\boldsymbol{x}_i\boldsymbol{\lambda}_i} ^2 g_i(\boldsymbol{x}_i(\boldsymbol{\lambda}), \boldsymbol{\lambda}_i; \hat{\xi}_i^{(m)})\bar{\boldsymbol{u}}_i^{(m-1)}$\;
    $\boldsymbol{u}_i^{(m)} \leftarrow\bar{\boldsymbol{u}}_i^{(m-1)}-\eta \partial _{\boldsymbol{x}_i}^2 g_i(\boldsymbol{x}_i(\boldsymbol{\lambda}), \boldsymbol{\lambda}_i; \hat{\hat{\xi}}_i^{(m)}) \bar{\boldsymbol{u}}_i^{(m-1)}$\;
}
\Return{$\boldsymbol{v}_i^{(M)} \approx \langle\mathrm{d}_{\boldsymbol{\lambda}} f  \rangle_i = \mathrm{d}_{\boldsymbol{\lambda}_i}\sum _{k} f_{k} (\boldsymbol{x}_{k} (\boldsymbol{\lambda}),\boldsymbol{\lambda} _{k})$}\;
\end{algorithm}
Our HGP is obtained simply by letting the $i$-th client compute the $i$-th block of $\boldsymbol{v}^{(m)}$ and $\boldsymbol{u}^{(m)}$, denoted by $\boldsymbol{v}_{i}^{(m)}$ and $\boldsymbol{u}_{i}^{(m)}$, respectively.
We also replace $\boldsymbol{\Theta}$ with the $S$-step Push-Sum (\cref{alg:pushsum}) which we denote by $\hat{\boldsymbol{\Theta}}$.

HGP described in \cref{alg:hgp} proceeds as follows.
For \(m=0\), any $i$-th client can locally compute $\boldsymbol{u}_{i}^{(0)} =\partial _{\boldsymbol{x}_{i}} f_{i} (\boldsymbol{x}_{i} (\boldsymbol{\lambda} ),\boldsymbol{\lambda} _{i} )$ and $\boldsymbol{v}_{i}^{(0)} =\partial _{\boldsymbol{\lambda} _{i}} f_{i} (\boldsymbol{x}_{i} (\boldsymbol{\lambda} ),\boldsymbol{\lambda} _{i}  )$ (\cref{remark:update}).
Suppose the $i$-th client knows $\boldsymbol{u}_{i}^{(m-1)}$, which is trivially true when $m=1$.
Then, the $i$-th client can compute the average $\bar{\boldsymbol{u}}_{i}^{(m-1)} = \langle \hat{\boldsymbol{\Theta}} \boldsymbol{u}^{(m-1)}\rangle_{i}$ in \cref{eq:recur_prox} by running Push-Sum (\cref{remark:pushsum}).
Because $\langle \partial _{\boldsymbol{x}}^{2} g  \rangle_{ii} =\mathbb{E}[ \partial _{\boldsymbol{x}_i\boldsymbol{\lambda}_i}^{2} g _{i}\left( \boldsymbol{x}_{i} (\boldsymbol{\lambda} ),\boldsymbol{\lambda} _{i};\xi_i \right)]$ can be computed locally (\cref{remark:update}), the $i$-th client can compute $\boldsymbol{u}_{i}^{(m)}$ once $\bar{\boldsymbol{u}}_{i}^{(m-1)}$ is obtained.
This can be performed for every $m\geq0$, and similarly for $\boldsymbol{v}_{i}^{(m)}$.
Note that in \cref{alg:hgp}, we replace expectations for $g_i$ with its finite sample estimates.

HGP described in \cref{alg:hgp} holds two desirable properties.
First, the communication for HGP is only performed by \cref{alg:pushsum}, enabling hyper-gradient estimation over time-varying directed networks.
Second, HGP enjoys small communication complexity of \(O(d_{\boldsymbol{x}})\) as HGP exchanges only \(O(d_{\boldsymbol{x}})\) sized vector through Push-Sum.

We conclude this section by presenting the theoretical upper bound of HGP's estimation error.
\begin{assumption} \label{ass:jacob}
Let $\xi=\{\xi_1,\ldots,\xi_n\}$  and $\hat{g} (\boldsymbol{x},\boldsymbol{\lambda}; \xi)=\sum_{k=1}^{n} g_{k} (\boldsymbol{x}_{k} ,\boldsymbol{\lambda} _{k} ;\xi _{k} ) $, there exists $\kappa_{\boldsymbol{x}}$ and $\kappa_{\boldsymbol{\lambda}}$ such that for every $\boldsymbol{\lambda}$ and $\xi$,
\begin{equation*}
\sigma_{\max}( \partial_{\boldsymbol{x}}^2\hat{g} - \partial _{\boldsymbol{x}}^2 g ) \leq \kappa_{\boldsymbol{x}}, \quad \sigma_{\max}( \partial_{\boldsymbol{x}\boldsymbol{\lambda}}^2\hat{g} - \partial _{\boldsymbol{x}\boldsymbol{\lambda}}^2 g ) \leq \kappa_{\boldsymbol{\lambda}}.
\end{equation*}
\end{assumption}
\begin{theorem} \label{thm:error}
    Under \crefrange{ass:connect}{ass:jacob}, there exist constants $0\leq \tau < 1$ and $\delta  >0$ such that with probability at least $1-\epsilon $,
\begin{equation*}
\frac{\left\Vert \boldsymbol{v}^{(M)} -\mathrm{d}_{\boldsymbol{\lambda}} f \right\Vert }{\left\Vert \partial _{\boldsymbol{x}}  f   \right\Vert } \leq \frac{4\mu}{\alpha }\left( 1+\frac{8\sqrt{n}}{\delta } \tau ^{S}\right)\sqrt{2\log\frac{2n(d_{\boldsymbol{x}}+d_{\boldsymbol{\lambda}})}{\epsilon }} + \frac{2\beta}{\eta \alpha^2 }\frac{8\sqrt{n}}{\delta } \tau ^{S} +\exp (-O(M)) ,
\end{equation*}
where $\mu =\sqrt{\kappa_{\boldsymbol{\lambda}}^2+\kappa_{\boldsymbol{x}}^2\frac{\beta^2}{\alpha^2}}$, $\alpha= \min_{{\boldsymbol{\lambda} ,\xi}}\left\{\sigma_{\min}( \partial_{\boldsymbol{x}}^2\hat{g})\right\}$, and $\beta= \max_{{\boldsymbol{\lambda} ,\xi}}\left\{\sigma_{\max}( \partial_{\boldsymbol{x}\boldsymbol{\lambda}}^2\hat{g})\right\}$, under the conditions
\begin{equation*}
    0 < \alpha < \nicefrac{1}{\eta}, \quad \text{and} \quad S \geq \nicefrac{\log \frac{\delta}{16\sqrt{n}} \frac{\eta \alpha}{1-\eta \alpha }}{\log \tau} .
\end{equation*}
$\exp (-O(M))$ denotes the exponentially diminishing term.
\end{theorem}

\section{Related Work} \label{sec:related_works}
This section compares our work with other hyper-gradient estimation methods in the context of distributed bilevel optimization~\citep{Gao2022,Chen2022DBO,Yang2022}.
In particular, we compare ours and their works in terms of the type of hyper-gradient to estimate, available communication networks, and the requirement of explicit Jacobian in computation and memory.

\paragraph{Type of the hyper-gradient to estimate}
We can categorize the existing hyper-gradients in distributed scenarios into three classes: GlobalGrad~\citep{Yang2022}, ClientGrad~\citep{Chen2022DBO}, and LocalGrad~\citep{Gao2022}.
GlobalGrad takes the derivative of the sum or average of the outer-costs, \(\sum_{k=1}^{n}f_k(\boldsymbol{x}_{k}(\boldsymbol{\lambda}_1,\ldots,\boldsymbol{\lambda}_n),\boldsymbol{\lambda}_i)\).
GlobalGrad thus considers how perturbation on the hyper-parameter of a client alters the overall performance of clients.
Our HGP belongs to this class.
ClientGrad also considers how a hyper-parameter of a client alters the consensus parameters.
However, different from GlobalGrad, it represents the derivative of a single client outer-cost, i.e., \(f_i(\boldsymbol{x}_{i}(\boldsymbol{\lambda}_1,\ldots,\boldsymbol{\lambda}_n),\boldsymbol{\lambda}_i)\).
Consequently, ClientGrad cannot represent the influence between a hyper-parameter and outer-cost across different clients.
LocalGrad considers the hyper-gradient of a single client's outer-cost, \(f_i(\boldsymbol{x}_{i}(\boldsymbol{\lambda}_i),\boldsymbol{\lambda}_i)\), and it assumes that the cost of a client never influences the optimal parameter of the others.
\citet{Gao2022} regards LocalGrad as GlobalGrad assuming the homogeneous client data distributions and utilizes classic hyper-gradient computation~\citep{ghadimi2018approximation}. 

Overall, GlobalGrad can be considered to have richer information than LocalGrad or ClientGrad.
This difference is particularly essential in influence estimation.
It is known that we can quantify the contribution of training instances~\citep{Koh2017} and clients~\citep{xue2021toward} to the overall performance using the hyper-gradient.
As GlobalGrad is the only hyper-gradient that can measure the effects on overall performance \(\sum_{k=1}^{n}f_k(\boldsymbol{x}_{k}(\boldsymbol{\lambda}_1,\ldots,\boldsymbol{\lambda}_n),\boldsymbol{\lambda}_i)\), the accurate estimation of GlobalGrad is essential for influence estimation.

\paragraph{Available communication networks}
For communication networks, previous approaches for GlobalGrad and ClientGrad require static undirected communication network\footnote{Their approaches will work with time-varying undirected networks if clients know the rows or columns of a doubly-stochastic matrix corresponding to a time-varying network at every time step. However, such a situation is not practically realized~\citep{Tsianos2012}.}. 
Only our HGP can operate over time-varying directed networks, with the help of Push-Sum. 
Note that we can obtain HGP for the centralized or static undirected networks by using appropriate estimators for the averaging operator $\boldsymbol{\Theta}$. 
For example, one can easily construct a centralized HGP by introducing $\boldsymbol{\Theta}$ that collects all the values to the central server and computes their average, similar to FedAVG~\citep{McMahan2017}.
HGP over static undirected networks can also be realized as a special case of the original HGP, by replacing the Push-Sum in \cref{alg:hgp} with consensus optimization using a constant doubly-stochastic mixing matrix. 
We demonstrate this flexibility in our experiment (\cref{sec:personal}).

\paragraph{Explicit use of Hessian and Jacobian Matrices}
As we noted in Table~\ref{table:dbo}, some distributed hyper-gradient computation requires communicating large matrices such as Hessian.
This also indicates that we need to compute and store Hessian matrices locally.
However, computing and storing such large matrices are undesirable in modern machine learning such as deep neural networks.
Our HGP is free from these difficulties as it requires only Hessian-vector-products and Jacobian-vector-products that can be computed without computing these large matrices.
 
\section{Experiments} \label{sec:exp}
In this section, we evaluate the estimation error of HGP (\cref{sec:est_err}), and demonstrate the efficacy of our HGP in two applications for decentralized FL, that are, decentralized influence estimation (\cref{sec:infl}) and personalization (\cref{sec:personal}).
Detailed settings and results of every experiment can be found in our appendix.

\begin{figure}[t]
    \centering
    \begin{subfigure}[b]{0.34\textwidth}
        \includegraphics[width=\textwidth]{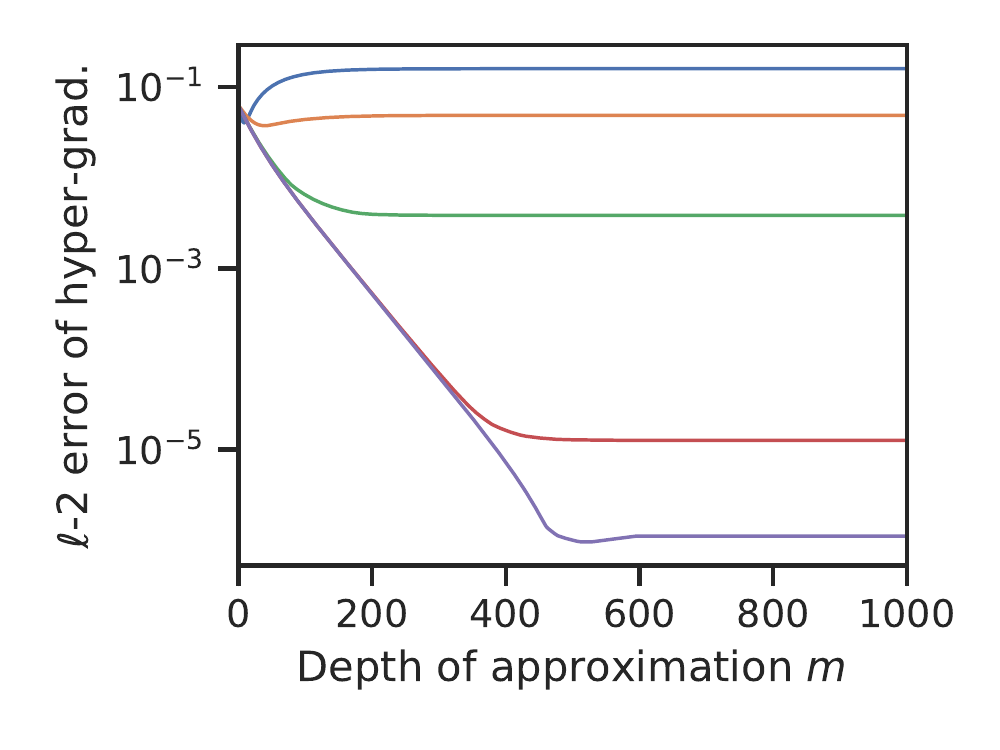}
        \caption{$\ell_2$ error with full-batch $g_i$}
        \label{fig:hypergrad_error_det}
    \end{subfigure}
    \begin{subfigure}[b]{0.34\textwidth}
        \includegraphics[width=\textwidth]{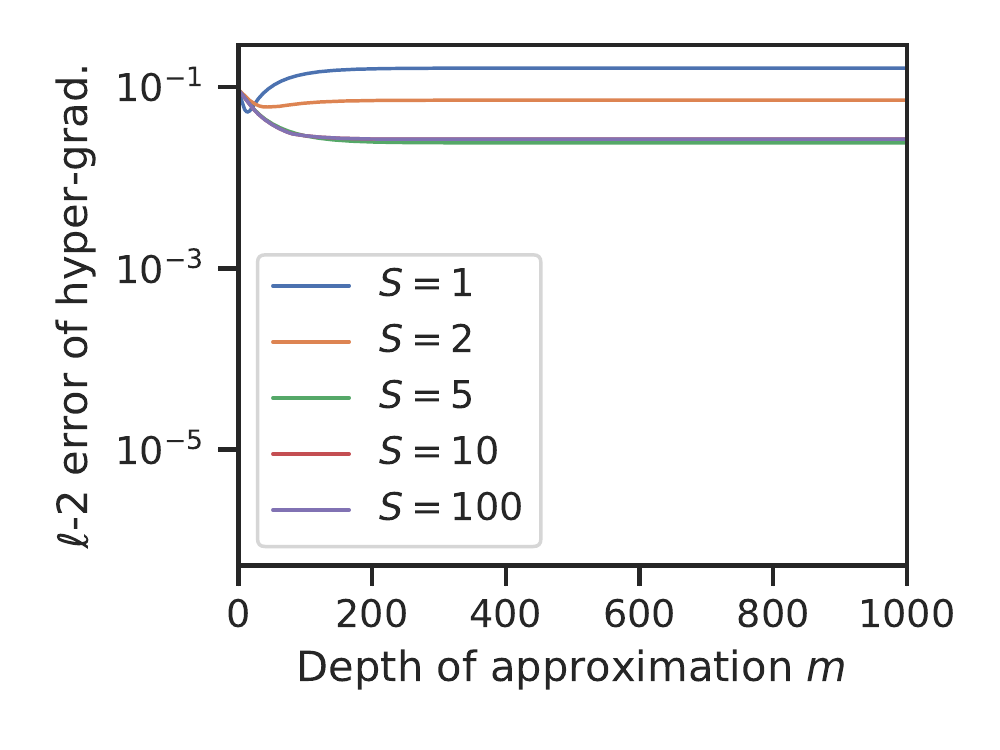}
        \caption{$\ell_2$ error with mini-batch $g_i$}
        \label{fig:hypergrad_error_sto}
    \end{subfigure}
    \begin{subfigure}[b]{0.3\textwidth}
        \includegraphics[width=\textwidth]{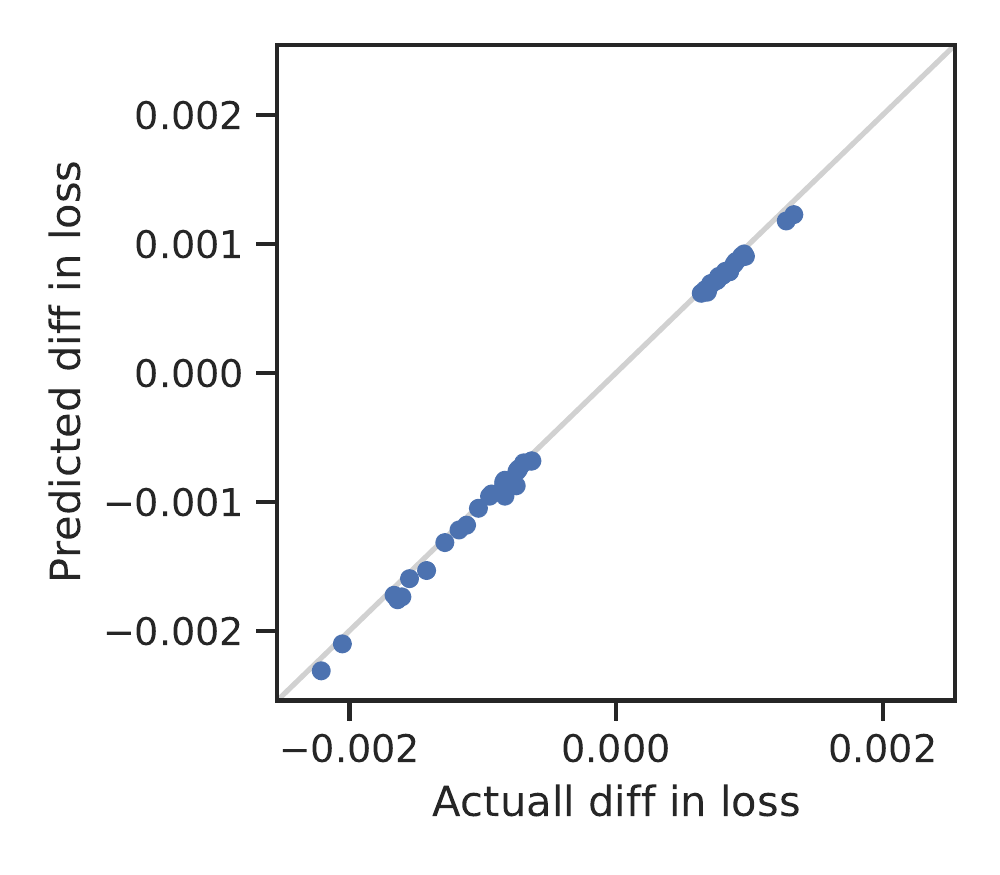}
        \caption{Error of influence estimation}
        \label{fig:lie_error}
    \end{subfigure}
    \caption{Estimation error of the hyper-gradient (a, b) and estimation error of influence instances (c) computed on logistic regression model trained for the synthetic classification dataset~\citep{Marfoq2021}.
    (a, b) show average $\ell_2$-error between true hyper-gradient $\mathrm{d}_{\boldsymbol{\lambda}}f$ and its estimation $ \boldsymbol{v}^{(m)}$ of the validation loss over 10 whole iterations of HGP with different seeds (error bars are omitted for visibility).
    While (a) uses data full-batch to compute Hessians/Jacobians, (b) uses data mini-batch.
    Both (a) and (b) plot the error varying the number of Push-Sum iterations $S$.
In (c), we selected the 50 most influential training instances, predicted by our estimation, and displayed the actual (horizontal) and estimated (vertical) changes in the validation loss after excluding each instance and retraining.}
\end{figure}
\vspace{-6pt}

\subsection{Estimation Error of Hyper-Gradient} \label{sec:est_err}
We assess the estimation error of the HGP estimate $\boldsymbol{v}^{(M)}$ as well as the impact of the choice of Neumann approximation order $M$ and the number of Push-Sum iterations $S$ to the error.

\noindent\textbf{Communication Network}\: We simulated time-varying directed networks with $n=3$.
We generated each directed edge $i \to j$ with probability $\rho_{ij}$ independently for each time step.
We also added self-loop edges, i.e., $\rho_{ii}=1$ for all $i$. 
For every $i \neq j$, $\rho_{ij}$ was independently sampled from the uniform distribution over $[0.4, 0.8]$.

\noindent\textbf{Dataset and Model}\: We generated a synthetic, heterogeneously distributed classification datasets following \citet{Marfoq2021}.
We used a linear logistic regression model and performed decentralized FL by using stochastic gradient push \citep{Assran2019} to obtain $\boldsymbol{x}_i(\boldsymbol{\lambda}) = \boldsymbol{x}_j(\boldsymbol{\lambda}), \forall i,j$.
For the client objective $g_i$, we used multiple regularization parameters similar to \citet{Pedreska2016}:
\begin{equation} \label{eq:lr_error}
g_{i}(\boldsymbol{x}_{i}, \boldsymbol{\lambda}_{i}; \xi_{i}) = \mathtt{BinaryCrossEntropy}\left(\xi_{i}^{\mathrm{out}}, \xi_{i}^{\mathrm{in}\top}\boldsymbol{x}_{i}\right) + \frac{1}{2}\boldsymbol{x}_{i}^{\top}\operatorname{diag}(\boldsymbol{\lambda}_{i})\boldsymbol{x}_{i},\ \forall i,
\end{equation}
where,  the instance of the client is the tuple of input and output $\xi_i = (\xi_{i}^{\mathrm{in}} \in \mathbb{R}^{5}, \xi_{i}^{\mathrm{out}} \in \{0,1\})$ and $\mathrm{diag}(\boldsymbol{\lambda}_{i})$ denotes the diagonal matrix consisting of the elements of $\boldsymbol{\lambda}_{i} \in \mathbb{R}^{d_{\boldsymbol{x}}}$. For the outer-cost $f_i$, we employed the binary cross-entropy loss computed on validation datasets.

\cref{fig:hypergrad_error_det,fig:hypergrad_error_sto} demonstrates the average estimation error $\left\Vert \boldsymbol{v}^{(m)} - \mathrm{d}_{\boldsymbol{\lambda}}  f  \right\Vert$ over 10 whole HGP iterations for various $M$ and $S$.
In \cref{fig:hypergrad_error_det}, we used the whole dataset to compute Hessian $\partial_{\boldsymbol{x_i}}^2 g_i$ and Jacobian $\partial_{\boldsymbol{x_i}\boldsymbol{\lambda_i}}^2 g_i$, while we used the data mini-batch in \cref{fig:hypergrad_error_sto}.
These results consistently align with \cref{thm:error}: the error reduces as $M$ and $S$ and increases in both figures\footnote{$S=1, 2$ seem to be exceptions. This is because $S$ needs to be larger than some constant in \cref{thm:error}.}; there is an irreducible error when the variance of Hessian and Jacobian matrices are non-zeros in \cref{fig:hypergrad_error_sto}.
The latter observations suggest that, when we use data mini-batch, the error tends to saturate and a small number of $M$ and $S$ would be sufficient.

\subsection{Decentralized Influence Estimation of Training Instances} \label{sec:infl}
We solve influence estimation \citep{Koh2017} in a decentralized manner. The goal of decentralized influence estimation is to allow any client to know how the removal of its local training instance changes the total model performance across all clients. 
HGP can easily realize such an estimation by letting $\boldsymbol{\lambda}_{i}$ be ones mask vector for the losses of training instances.
Let $\xi_{ik}$ be the $k$-th instance in the $i$-th client training dataset.
We then consider the following local cost:
\begin{equation} \label{eq:infl}
    g_{i}(\boldsymbol{x}_{i}, \boldsymbol{\lambda}_{i}; \xi_{ik}) = \langle{\boldsymbol{\lambda}_{i}}\rangle_k \ell(\xi_{ik}, \boldsymbol{x}).  
\end{equation}
where  $\ell$ is a loss function computed from the instance $\xi_{ik}$.
By using \cref{eq:infl}, the $k$-th element of $\langle -\mathrm{d}_{\boldsymbol{\lambda}}  f  \rangle_{i}$ gives the linear approximation of the changes in $ f $ by the removal of $\xi_{ik}  $ from the training dataset. 
We used the same settings as \cref{sec:est_err} except for the definition of $g_i$.

\cref{fig:lie_error} plots the approximation and actual change of $ f   $.
We selected the 50 elements with the largest absolute values in $\langle -\mathrm{d}_{\boldsymbol{\lambda}}  f  \rangle_{i} $ to determine the top-50 influential instances.
We then computed the actual effect of removing a single data instance for each of these 50 instances.
\cref{fig:lie_error} demonstrates that HGP accurately predicted the change in validation loss.
We note that, for the first time, HGP enabled the decentralized influence estimation over a time-varying network.

\subsection{Decentralized Personalization for Deep Learning} \label{sec:personal}
This section demonstrates the ability of HGP to solve personalization as a bilevel optimization problem. 
Although our theoretical guarantee in \cref{thm:error} is restricted to strongly convex cases, however, in this section, we experimentally demonstrate that HGP can also be effectively applied to deep neural networks.
We present two formulations of the bilevel problem and benchmarked their performance on four different communication networks. We utilized a 10\% subset of the EMNIST dataset consisting of 62 classes.
Preprocessing was performed for data heterogeneity among clients and let clients train CNN classification models.

\paragraph{Settings} We introduce two personalization methods, HGP-PL and HGP-MTL, derived from different $g_i$ formulations.
HGP-PL trains an attention mask layer using hyper-gradients which can be regarded as a family of personalized layer technique~\citep{Arivazhagan2019}.
HGP-PL optimizes $\boldsymbol{\lambda}_i \in \mathbb{R}^{62}$ to obtain a class-wise attention mask vector \(\mathtt{Softmax}(\boldsymbol{\lambda}_{i})\in [0,1]^{62}\) which is then multiplied by the logits from a CNN, parameterized by \(\boldsymbol{x_i}\).
HGP-MTL lets each client train an ensemble classifier that produces weighted average predictions from three base CNNs,  similarly to FedEM~\citep{Marfoq2021}.
HGP-MTL trains $\boldsymbol{x}_i$ as base CNN parameters and optimizes $\boldsymbol{\lambda}_i \in \mathbb{R}^{3}$ to determine the ensemble weight $\mathtt{Softmax}(\boldsymbol{\lambda}_{i})\in [0,1]^{3}$ of the base CNN outputs.
We utilized the training dataset's cross-entropy for $f_i$.
We performed five Adam~\citep{Adam} outer-steps using $\boldsymbol{v}_i^{(M)}$ to update $\boldsymbol{\lambda}_i$ from the zeros initial vector, and ran decentralized FL training for $\boldsymbol{x}_i$ before every outer-step.
Practical modifications on HGP were also incorporated: utilizing a single estimate \(\hat{\xi}_{i}\) for updates of both \(\boldsymbol{v}_i^{(m)}\) and \(\boldsymbol{u}_i^{(m)}\), and maintaining the updated debias weight $\omega_i^{(s)}$ without initialization in \cref{alg:pushsum}.
We selected $M=10$ and $S=10$ from the observation in \cref{sec:est_err}\footnote{The results in our appendix indicate that $S=1$ still produces comparable performance with the aforementioned modifications on HGP.}.

To demonstrate the flexibility of HGP in communication networks, we simulated four different networks, fully-connected (\texttt{FC}), static undirected (\texttt{FixU}), random undirected (\texttt{RandU}), and random directed (\texttt{RandD}) networks.
\texttt{FC} allows clients to communicate with all the others at any time step, i.e.,  $\mathcal{G}(s)$ is the fully-connected graph for all \(i, j\) and $s>0$.
\texttt{FixU} is a static undirected network simulated by a binomial Erd\H{o}s-R\'enyi graph~\citep{erdos59a} adding the self-loop edges. We construct a constant doubly-stochastic mixing matrix by the fast-mixing Markov chain~\citep{Boyd03fastestmixing} rule.
In \texttt{RandU} and \texttt{RandD}, we followed the setting in \cref{sec:est_err}.
For \texttt{RandU}, we set $\mathcal{N}_{i}^{\mathrm{out}}(s)=\mathcal{N}_{i}^{\mathrm{in}}(s),\forall i,s$ to realize undirected edges.
All communication scenarios were set with \(n=100\) clients.

We compared our approaches with baselines on each communication setting.
For \texttt{FC} and \texttt{FixU} settings, we compared several personalization approaches: a personalized model trained only on the local dataset (Local), FedAvg with local tuning (FedAvg+)~\citep{Jiang2019}, Clustered-FL~\citep{Sattler2020}, pFedMe~\citep{Dinh2020}, and centralized and decentralized versions of FedEM \citep{Marfoq2021}.
We also trained the global models using FedAvg~\citep{McMahan2017}, SGP~\citep{Nedic2016}, and FedProx~\citep{Fedprox}.
Since SGP is mathematically equivalent to FedAvg on \texttt{FC}, we treat them as equivalent approaches.
For a fair comparison, we performed hyper-parameter tuning and reported the best results.
All baselines and HGP inner-optimizations on \texttt{FC} and \texttt{FixU} ran a training procedure following \citet{Marfoq2021}, and ran SGP~\citep{Nedic2016} on \texttt{RandU} and \texttt{RandD}.

\begin{table*}[t]
    \setlength{\tabcolsep}{5pt}
    \vskip -0.1in
	\caption{Test accuracy of personalized models on EMNIST (average clients / bottom 10\% percentile).}
	\label{table:exp}
	\centering
    \scriptsize
\begin{tabular}{clcccc}
	\toprule
	& \multirow{3}{*}{Method}      &  \multicolumn{4}{c}{Communication network}    \\ \cmidrule(lr){3-6}
	                              &                              & Fully-Connected  & Static Undirected   & \textbf{Time-Varying Undirected}  & \textbf{Time-Varying Directed} \\ 
	                              &                              &  (\texttt{FC}) &  (\texttt{FixU}) & (\textbf{\texttt{RandU}}) &  (\textbf{\texttt{RandD}}) \\ \midrule
	\multirow{2}{*}{Global}       & SGP/FedAvg             & \(82.9\) / \(74.5\) & \(83.8\) / \(76.8\) & \(80.0\) / \(71.8\) & \(80.1\) / \(72.4\)         \\
	                              & FedProx                & \(82.9\) / \(74.5\) & n/a & n/a & n/a \\ \midrule
	\multirow{7}{*}{Personalized} & Local                  & \(74.7\) / \(63.0\) & \(74.7\) / \(63.0\) & \(73.4\) / \(64.0\) & \(73.4\) / \(64.0\)             \\
	                              & FedAvg+                & \(83.2\) / \(74.3\) & n/a & n/a & n/a \\
	                              & Clustered-FL           & \(83.0\) / \(74.6\) & n/a & n/a & n/a \\
	                              & pFedMe                 & \(85.1\) / \(79.0\) & n/a & n/a & n/a \\
	                              & FedEM                  & \(84.0\) / \(75.4\) & \(83.9\) / \(75.7\) & n/a & n/a      \\
	                              & \textbf{HGP-MTL (Ours)}    & \(84.0\) / \(76.1\) & \(84.0\) / \(76.3\) & \(82.6\) / \(74.3\) & \(81.8\) / \(72.9\) \\
	                              & \textbf{HGP-PL (Ours)}     &  \(\boldsymbol{88.1}\) / \(\boldsymbol{82.6}\) & \(\boldsymbol{88.0}\) / \(\boldsymbol{82.2}\) & \(\boldsymbol{85.5}\) / \(\boldsymbol{79.6}\) & \(\boldsymbol{85.6}\) / \(\boldsymbol{79.6}\) \\  \bottomrule
\end{tabular}
\vskip -0.1in
\end{table*}

\vspace{-3pt}

\paragraph{Results} \cref{table:exp} presents the average test accuracy with weights proportional to local test dataset sizes.
Here, HGP-PL performed the best across all communication settings demonstrating the effectiveness of the estimated hyper-gradients. HGP-MTL attained similar improvements to FedEM reflecting the fact that both are the mixture of three CNNs.
We also investigated whether the accuracy gain was shared among all clients.
Table~\ref{table:exp} shows the accuracy of the bottom 10\% percentile of the clients.
All our approaches improved accuracy at the 10\% percentile from the baselines in all communication settings, confirming that the clients fairly benefited from our personalization.
In terms of the available communication networks, our approaches with HGP are the only methods that are applicable to time-varying networks, and they attained superior performance compared to the currently available baselines, SGP and Local.

\section{Conclusion}
In conclusion, this paper addressed the challenge of estimating hyper-gradients in decentralized FL, namely, the communication cost and the availability over time-varying directed networks.
We proposed an alternative equivalent optimality condition of decentralized FL based on iterations of Push-Sum, free of explicit consensus constraints.
Consequently, the hyper-gradient derived from our optimality condition only requires (i) communication of vectors by Push-Sum and (ii) communications over time-varying directed networks.
The convergence of our HGP toward the true hyper-gradient was validated both theoretically and empirically.
Furthermore, we demonstrated the practical benefits of HGP in two novel applications: decentralized influence estimation and personalization over time-varying communication networks.

\section*{Acknowledgment}
Satoshi Hara is partially supported by JST, PRESTO Grant Number JPMJPR20C8, Japan.
Computational resource of AI Bridging Cloud Infrastructure (ABCI) provided by National Institute of Advanced Industrial Science and Technology (AIST) was used to produce the experimental results.

{

\bibliographystyle{abbrvnat}
\bibliography{main}
}

\newpage
\appendix

\section*{Limitation and Future Work}
\cref{thm:error} indicates estimation bias remains due to the variance from $g_i$ and finite numbers of $M$ and $S$.  
Although the focus of this paper is a hyper-gradient estimation, analyzing the convergence of bilevel optimizations using such biased gradients is a promising direction for future work.
Additionally,  typical deep learning scenarios with non-convex  $g_i$ violate \cref{ass:convex}.
Although we demonstrated our performance with deep learning models in \cref{sec:personal}, extending \cref{thm:error} to non-convex settings should be included in another future work.

\section{Detailed Experimental Settings and Results}
This section describes detailed settings and additional results of experiments in \cref{sec:exp}.
The source code for reproduction is provided in separate files.

\subsection{Estimation Error of Hyper-Gradient} \label{app:est_err}
We employed synthetic and heterogeneously distributed classification datasets following \citet{Marfoq2021}[I.1.5].
We produced binary classification datasets with five-dimensional inputs, utilizing three as the count of underlying distributions and a constant 0.4 for Dirichlet distribution.

Every client holds training and validation datasets, each with 100 instances.
We explored two types of $g_i$: full-batch $g_i$ and mini-batch $g_i$.
Full-batch $g_i$ is deterministic, utilizing the complete batch, whereas mini-batch $g_i$ is stochastic, using only 20 instances per mini-batch.
For both $g_i$ types, we conducted 5000 iterations using the SGP~\citep{Assran2019} with a multi-step learning rate scheduler to obtain $\boldsymbol{x}(\boldsymbol{\lambda})$.

We plotted the complete versions of \cref{fig:hypergrad_error_det} and \cref{fig:hypergrad_error_sto} with error bars as \cref{fig:hypergrad_error_det_err_bar} and \cref{fig:hypergrad_error_sto_err_bar}, respectively.

\begin{figure}[t]
    \centering
    \begin{subfigure}[b]{0.49\textwidth}
        \includegraphics[width=\textwidth]{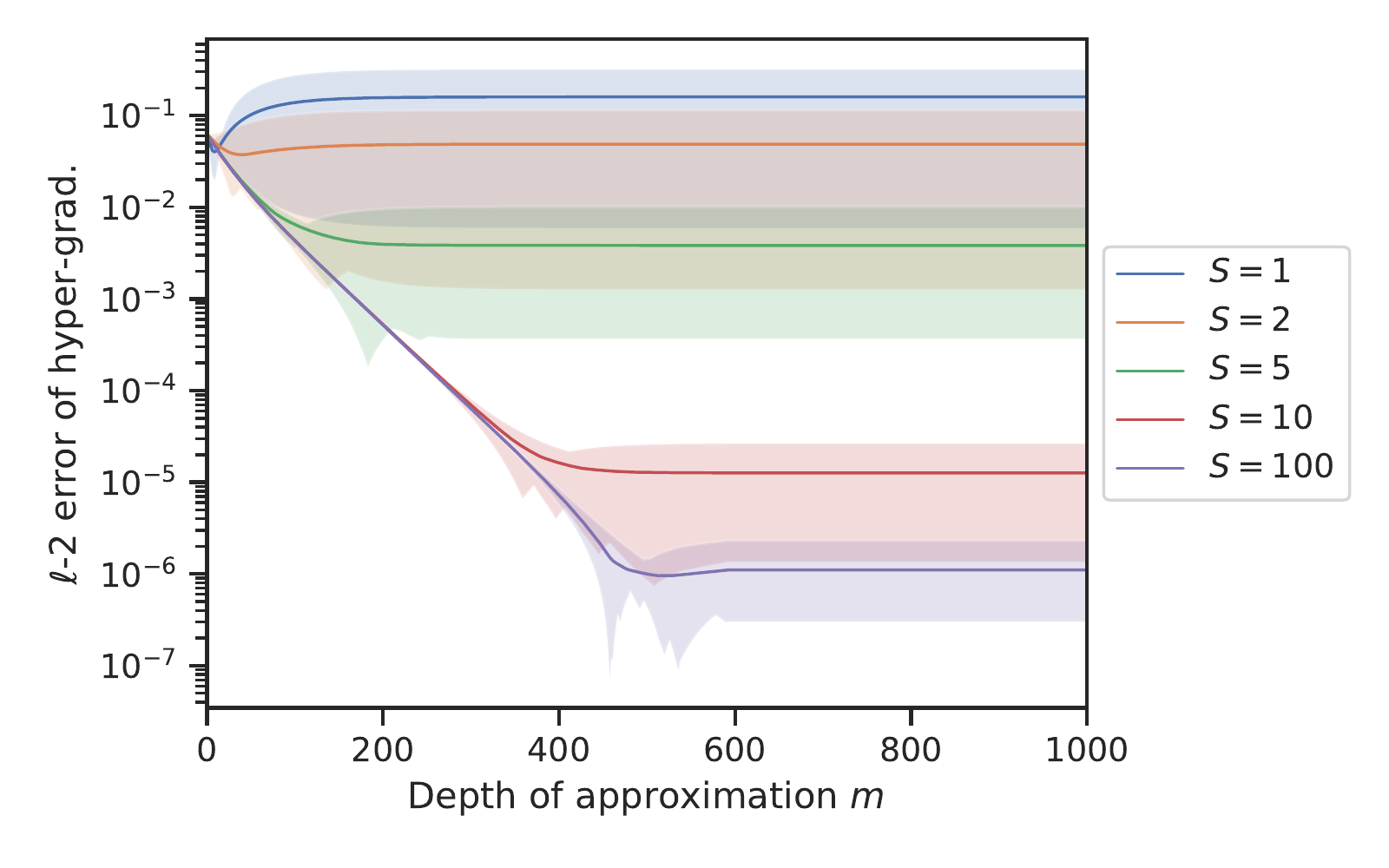}
        \caption{Full-batch $g_i$}
        \label{fig:hypergrad_error_det_err_bar}
    \end{subfigure}
    \begin{subfigure}[b]{0.49\textwidth}
        \includegraphics[width=\textwidth]{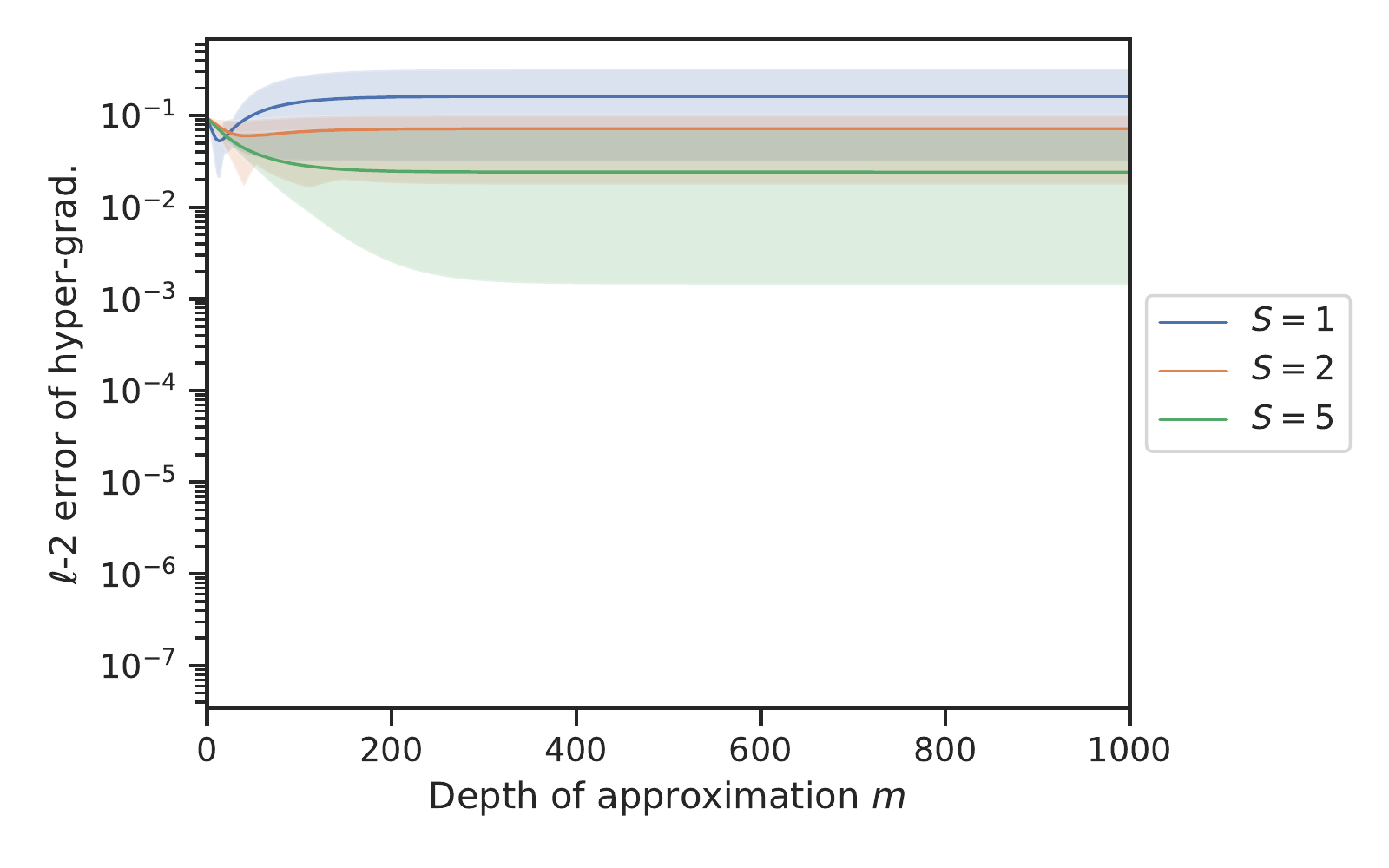}
        \caption{Mini-batch $g_i$}
        \label{fig:hypergrad_error_sto_err_bar}
    \end{subfigure}
    \caption{The estimation errors of the hyper-gradient computed on a logistic regression model trained for the synthetic classification dataset~\citep{Marfoq2021}.
Figures indicate the average $\ell_2$-error between the true hyper-gradient $\mathrm{d}_{\boldsymbol{\lambda}}f$ and its estimation $\boldsymbol{v}^{(m)}$ of the validation loss over 10 iterations of HGP with different seeds.
While (a) computes Hessians/Jacobians using full-batch, (b) uses mini-batch.
We excluded the results with $S=10,100$ from (b) as they have no significant differences compared to the result with $S=5$.
The top and bottom of the error bars represent the 10\% and 90\% interpolated percentile of the estimation errors over 10 whole HGP iterations.}
\end{figure}
\vspace{-6pt}

\subsection{Decentralized Influence Estimation of Training Instances} \label{app:infl}
\paragraph{Detailed and additional settings}
In addition to the case of convex $g_i$ discussed in \cref{sec:infl}, i.e., logistic regression model with binary cross-entropy loss, we also conducted influence estimation on possibly non-convex settings by using a CNN classifier trained on distributed EMNIST datasets. 

For the convex case, we utilized the same dataset as described in Section \ref{sec:est_err}. We employed full-batch loss for $g_i$ and utilized SGP with the same setting as presented in \cref{app:est_err} to obtain $\boldsymbol{x}(\boldsymbol{\lambda})$. 
We used $M=500$, $S=100$, and $\eta=1$ for HGP.

The CNN classifier adopts the architecture described in Table \ref{table:cnn_arch}, incorporating L2 regularization with a decay rate of $0.001$. 
We utilized 1\% of all instances in the EMNIST dataset, specifically those belonging to digit classes. 
We followed the procedure outlined in \cref{app:personal} to split the subset of EMNIST for clients.
As in \cref{sec:infl} we employed $g_i$ in which each softmax cross-entropy loss for an instance $\xi_{ik}$ is multiplied by $\boldsymbol{\lambda}_{i,k}=1$. 
During both decentralized FL training and influence estimation for the CNN classifier, we employed mini-batch loss with a batch size of 512, as well as full-batch loss for $g_i$. 
In both settings of $g_i$, we executed 5000 iterations of SGP and conducted HGP with $M=1000$, $S=100$, and $\eta=0.05$.

\begin{table}[t]
	\caption{CNN architecture for experiments in \cref{app:infl}. The input size and output size are specified in the format: channels $\times$ height $\times$ width.}
	\label{table:cnn_arch}
	\centering
    \small
\begin{tabular}{lcc}
  \toprule
  Torch Operation & Input Size & Output Size \\
  \midrule
   \texttt{Conv2d(1, 20, 5)} & $1 \times 28 \times 28$ & $20 \times 24 \times 24$ \\
   \texttt{ReLU} & $20 \times 24 \times 24$ & $20 \times 24 \times 24$ \\
   \texttt{MaxPool2d(2)} & $20 \times 24 \times 24$ & $20 \times 12 \times 12$ \\
   \texttt{Conv2d(20, 20, 5)} & $20 \times 12 \times 12$ & $20 \times 8 \times 8$ \\
   \texttt{ReLU} & $20 \times 8 \times 8$ & $20 \times 8 \times 8$ \\
   \texttt{MaxPool2d(2)} & $20 \times 8 \times 8$ & $20 \times 4 \times 4$ \\
   \texttt{View(-1, 4 * 4 * 20)} & $20 \times 4 \times 4$ & $320$ \\
   \texttt{Linear(320, 10)} & $320$ & $10$ \\
  \bottomrule
\end{tabular}
\end{table}

\paragraph{Additional results}
The estimation results for full-batch $g_i$ and mini-batch $g_i$ under non-convex settings using a CNN classifier are displayed in \cref{fig:app_infl_err_emnist,fig:app_infl_err_emnist_sto}.
Compared to convex $g_i$ (\cref{fig:app_infl_err_synth}), the non-convex results appear to be noisier.

We performed further exploration of non-convex cases through R2 and F1 scores, measuring the correlation of predictions and discriminatory capability for negative influence instances, respectively.

The R2 score measures the correlation between actual and predicted differences in losses (influences) caused by the removal of the most influential instances.
The R2 score follows a definition of $R^2=1-{SS}_{\mathrm{res}}/{SS}_{\mathrm{tot}}$, where ${SS}_{\mathrm{res}}$ denotes the residual sum of squares score and ${SS}_{\mathrm{tot}}$ the total sum of squares.
One can see from \cref{table:infl} that the predicted influences under non-convex settings present a mild yet positive correlation to the actual influences.

The F1 score studies the ability of estimated hyper-gradients to differentiate between negative and positive influences.
An instance is deemed to have a negative influence if its removal and retraining decrease the validation loss.
The F1 score is computed by considering the influence estimation as a binary classification task, in which prediction is true-positive when both prediction and actual removal demonstrate a negative influence.
The F1 scores under non-convex settings in \cref{table:infl} indicate that the discriminatory capacity for negative influence instances is fairly accurate, confirming the effectiveness of our approach in non-convex settings.

\begin{figure}[t]
    \centering
    \begin{subfigure}[b]{0.32\textwidth}
        \includegraphics[width=\textwidth]{figs/lie_error.pdf}
        \caption{Convex \& full-batch $g_i$}
        \label{fig:app_infl_err_synth}
    \end{subfigure}
    \begin{subfigure}[b]{0.32\textwidth}
        \includegraphics[width=\textwidth]{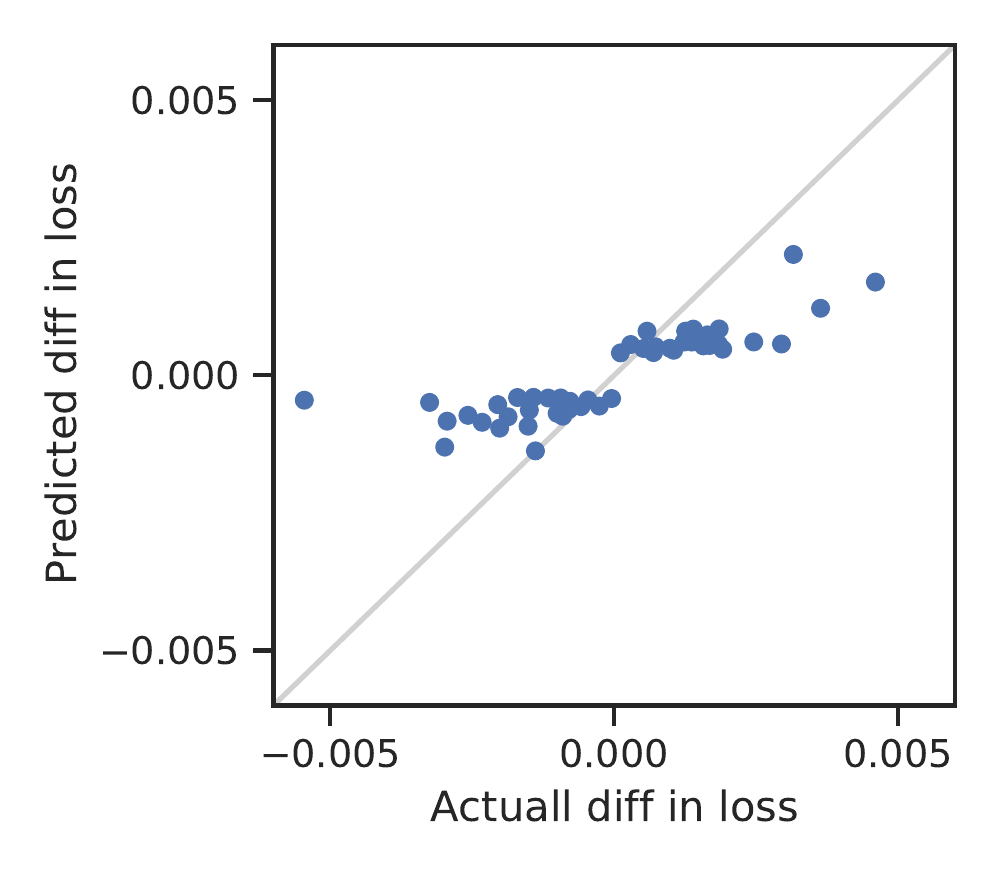}
        \caption{Non-convex \& full-batch $g_i$}
        \label{fig:app_infl_err_emnist}
    \end{subfigure}
     \begin{subfigure}[b]{0.32\textwidth}
        \includegraphics[width=\textwidth]{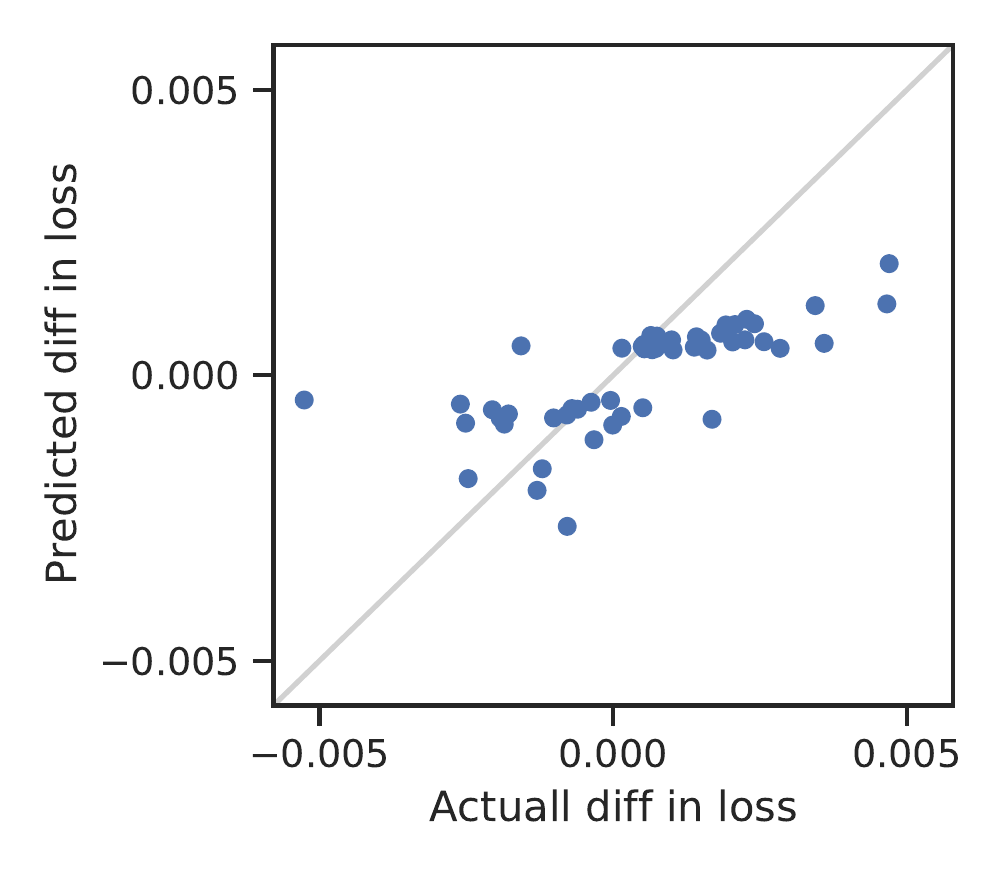}
        \caption{Non-convex \& mini-batch $g_i$}
        \label{fig:app_infl_err_emnist_sto}
    \end{subfigure}
    \caption{Estimation error of influence instances computed on  (a)  logistic regression model trained for the synthetic classification dataset~\citep{Marfoq2021} and (b, c) CNN classifier trained for the subset of EMNIST.
We selected the 50 most influential training instances, predicted by our estimation, and displayed the actual (horizontal) and estimated (vertical) changes in the validation loss after excluding each instance and retraining.} \label{fig:app_infl}
\end{figure}

\begin{table}[t]
	\caption{Quantitative evaluation of influence estimation}
	\label{table:infl}
	\centering
    \small
\begin{tabular}{llccc}
	\toprule
 Model & Dataset    &  Mini-batch $g_i$?          &  R2 Score & F1 Score  \\\midrule
Logistic Regression & Synthetic dataset & No   & 0.99                     & 1.00                      \\ 
	CNN & Subset of EMNIST                  & No &  0.53                   & 1.00                   \\ 
	CNN & Subset of EMNIST                  & Yes &  0.41                   & 0.91                   \\  \bottomrule
\end{tabular}
\end{table}

\subsection{Decentralized Personalization for Deep Learning} \label{app:personal}
\paragraph{Dataset and Model}
In our experiments, we used 10\% of the EMNIST dataset as in \citet{Marfoq2021}.
We generated a personalized version of EMNIST involving a combination of heterogeneity in the distribution in the following fashion.
First, clusters are equally separated into three clusters that use mean and variance to normalize inputs, generating heterogeneous input distributions.
For each cluster, we sampled mean and variance values independently from the uniform distribution with range \([{0,1}]\).
Second, instances are distributed across clients using \(n\)-dimensional Dirichlet distribution of parameter \(\alpha=0.4\) for each label, modeling heterogeneous label distributions. 

Any client splits the assigned dataset set of EMNIST by randomly selecting 20\% of instances to generate the test dataset.
The remaining instances are split into train and validation datasets with a ratio of 3:1.
We use the validation dataset only for the early stopping in outer-optimization of \text{HGP-PL} and \text{HGP-MTL}.

We trained the CNN with the architecture described in \cref{table:cnn_arch_emnist} for all the baselines with a single model and \text{HGP-PL}, and for base-predictor of \text{FedEM} and \text{HGP-MTL}.

\begin{table}[t]
	\caption{CNN architecture for the EMNIST dataset used in \cref{sec:personal}. The input size and output size are specified in the format: channels $\times$ height $\times$ width.}
	\label{table:cnn_arch_emnist}
	\centering
    \small
\begin{tabular}{lcc}
  \toprule
  Torch Operation & Input Size & Output Size \\
  \midrule
   \texttt{Conv2d(1, 32, 5)} & $1 \times 28 \times 28$ & $32 \times 24 \times 24$ \\
   \texttt{MaxPool2d(2, 2)} & $32 \times 24 \times 24$ & $32 \times 12 \times 12$ \\
   \texttt{Conv2d(32, 64, 5)} & $32 \times 12 \times 12$ & $64 \times 8 \times 8$ \\
   \texttt{MaxPool2d(2, 2)} & $64 \times 8 \times 8$ & $64 \times 4 \times 4$ \\
   \texttt{View(-1, 64 * 4 * 4)} & $64 \times 4 \times 4$ & $1024$ \\
   \texttt{Linear(1024, 2048)} & $1024$ & $2048$ \\
   \texttt{Linear(2048, 62)} & $2048$ & $62$ \\
  \bottomrule
\end{tabular}
\end{table}

\paragraph{Training Procedure}
Our \text{HGP-MTL} and \text{HGP-PL} consists of inner- and outer-optimizations.
In our approaches, clients perform decentralized FL as the inner-optimization to obtain $\boldsymbol{x}_i(\boldsymbol{\lambda})$, and then update $\boldsymbol{\lambda}_i$ using the hyper-gradient as the outer-optimization.
We refer to a set of  inner- and outer optimization as an outer-step, and clients perform several outer-steps to complete the update of  $\boldsymbol{\lambda}_i$.

For a fair comparison, in \texttt{FC} and \texttt{FixU} we followed the same setup and procedure in \citet{Marfoq2021} for the inner-optimization.
In \texttt{RandU} and \texttt{RandD}, we performed 600 steps of SGP~\cite{Nedic2016} as the inner-optimization.
SGP iterations employ learning rates of 0.1 for \text{HGP-PL} and the baselines (SGP and Local) and 0.5 for \text{HGP-MTL}.
Learning rates were scheduled to be multiplied by \(0.1\) at the 500-th and 550-th steps.
We set batch size to \(128\) and L2 regularization decay for the inner-parameter to \(0.001\). 

For all \(i\) in the outer-problem, we ran five outer-steps of Adam~\citep{Adam} iterations with \((\beta_1, \beta_2) = (0.9, 0.999)\) from the initial zeros hyperparameters  for both \text{HGP-PL} and \text{HGP-MTL}.
For Adam optimizer, we adopted different learning rates shown in \cref{table:hgp_setting}~(Hyper-learning rate).
For all the approaches, communication environments, and \(i\), \(f_i\) is the sum of the average cross-entropy loss over the local train dataset of the \(i\)-th client and L2 regularization loss of \(\boldsymbol{\lambda}_{i}\) with the rates shown in \cref{table:hgp_setting}~(L2 reg. rate).
We reported the mean test accuracy of an intermediate step that had maximum validation accuracy (i.e., early stopping).

\begin{table}[t]
	\caption{Parameters for the outer-problems in \cref{sec:personal}}
	\label{table:hgp_setting}
	\centering
    \small
\begin{tabular}{clll}
	\toprule
	Network                                          & Method            &  L2 regularization decay & Hyper-learning rate \\\midrule
	\multirow{2}{*}{\texttt{FC} and \texttt{FixU} }  & \text{HGP-PL}      & 0.01                     & 0.1                      \\ 
	                                                 & \text{HGP-MTL}     & 0.01                   & 1.0                      \\ \midrule
	\multirow{2}{*}{\texttt{RandU} and \texttt{RandD}} & \text{HGP-PL}     & 0.001                      & 0.1                   \\
	                                                 & \text{HGP-MTL}    &  0.01                   & 0.1                   \\  \bottomrule
\end{tabular}
\end{table}

\paragraph{Hyperparameter Tuning}
The learning rate for every baseline and our approach was set by performing a grid search on \(10^{-0.5}\), \(10^{-1}\), \(10^{-1.5}\), \(10^{-2}\), \(10^{-2.5}\), and  \(10^{-3}\).
The penalty parameter \(\mu\) for FedProx and pFedMe was optimized by grid search for \(10^{1}\), \(10^{0}\), \(10^{-1}\), and \(10^{-2}\).
We adopted \(\mu=10^{-2}\) for FedProx and \(\mu=10^{0}\) for pFedMe.
Clustered-FL used the same tolerances adopted in the official implementation~\citep{Sattler2020}.
For each method, \cref{table:lr} shows the learning rate that yielded the corresponding results in \cref{table:exp}.

\begin{table}[t]
	\caption{Learning rates used for the experiments in \cref{table:exp}. Base 10 logarithms are reported}
	\label{table:lr}
	\centering
\begin{tabular}{lcc}
	\toprule
	Method &  Fully-Connected (\texttt{FC}) & Static Undirected (\texttt{FixU}) \\ \midrule
	 \text{SGP/FedAvg}   & -1.5 & -1.0 \\
 \text{FedProx}          & -0.5 & n/a \\    
 \text{Local}            & -1.5 & 1.5 \\    
 \text{FedAvg+}          & -1.5 & n/a \\    
 \text{Clustered-FL}     & -1.5 & n/a \\    
 \text{pFedMe}           & -1.0 & n/a \\    
 \text{FedEM}            & -1.0 & -1.0 \\    
\text{HGP-PL}           & -1.5 & -1.5 \\
\text{HGP-MTL}          & -1.0 & -1.0 \\\bottomrule
\end{tabular}
\end{table}

\paragraph{Computational resources}
Our experiments were performed on a CPU/GPU cluster, specifically using Nvidia Tesla V100-SXM2-32GB or Tesla V100-SXM2-32GB GPUs.
\cref{table:compute} illustrates the time duration required for the entire outer-steps in each scenario.
The large time consumption for HGP-PL and HGP-MTL in \texttt{RandU} and \texttt{RandD} communication networks is primarily attributable to the inner-optimization of SGP.
This increased time requirement largely stems from the suboptimal efficiency of our SGP implementation.
We can potentially reduce the computation time on \texttt{RandU} and \texttt{RandD} to those of the \texttt{FC} and \texttt{FixU} communication networks, owing to their more efficient implementation of inner-optimization which are utilized in \citet{Marfoq2021}.

\begin{table}[t]
	\caption{Computation time for our decentralized personalization in \cref{sec:personal}}
	\label{table:compute}
	\centering
    \small
\begin{tabular}{llc}
	\toprule
	Network                                          & Method            &  Simulation time \\\midrule
	\multirow{2}{*}{\texttt{FC}}  & \text{HGP-PL}                          & 1h59min                 \\ 
	                                                 & \text{HGP-MTL}     &  2h47min                    \\
                                                      \midrule
	\multirow{2}{*}{\texttt{FixU}}  & \text{HGP-PL}                     &  2h33min                   \\ 
	                                                 & \text{HGP-MTL}     & 3h33min                  \\
                                                  \midrule
	\multirow{2}{*}{\texttt{RandU}} & \text{HGP-PL}                      & 12h56min                 \\
	                                                 & \text{HGP-MTL}    &  18h3min               \\ 
	                                                  \midrule
\multirow{2}{*}{\texttt{RandD}} & \text{HGP-PL}                         & 14h22min                      \\
	                                                 & \text{HGP-MTL}    & 19h21min                           \\ 
                                                  \bottomrule
\end{tabular}
\end{table}

\paragraph{Learning curve of outer-optimization}
\cref{fig:val_accs} plots the average validation accuracy at each outer-step of HGP-PL and HGP-MTL.
\cref{fig:val_accs} highlights the positive correlation between validation accuracy and the number of outer-steps, confirming that outer-opimizations properly worked and hyper-gradient estimations were sufficiently accurate.

As mentioned in \cref{sec:personal}, we found the performance of outer-optimization with $S=10$ and $S=1$ to be nearly identical (\cref{fig:fixu,fig:randu,fig:randd}).
This suggests that the HGP iteration complexity of $O(MS)$ can potentially be reduced to $O(M)$ for bilevel optimization purposes.

\begin{figure}[htp]
    \centering
    \begin{subfigure}[b]{0.49\textwidth}
        \includegraphics[width=\textwidth]{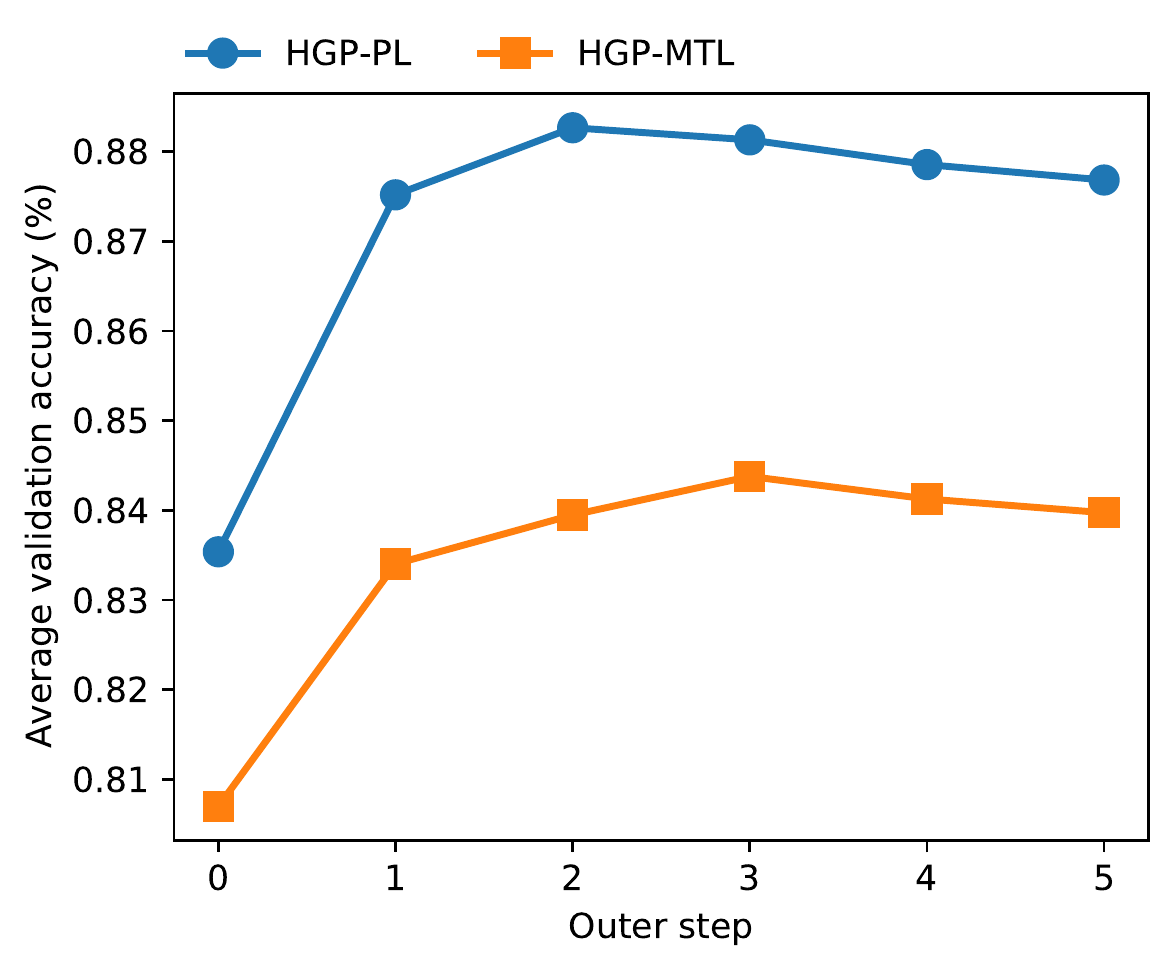}
        \caption{Fully-Connected (\texttt{FC})}
        \label{fig:fc}
    \end{subfigure}
    \hfill
        \begin{subfigure}[b]{0.49\textwidth}
        \includegraphics[width=\textwidth]{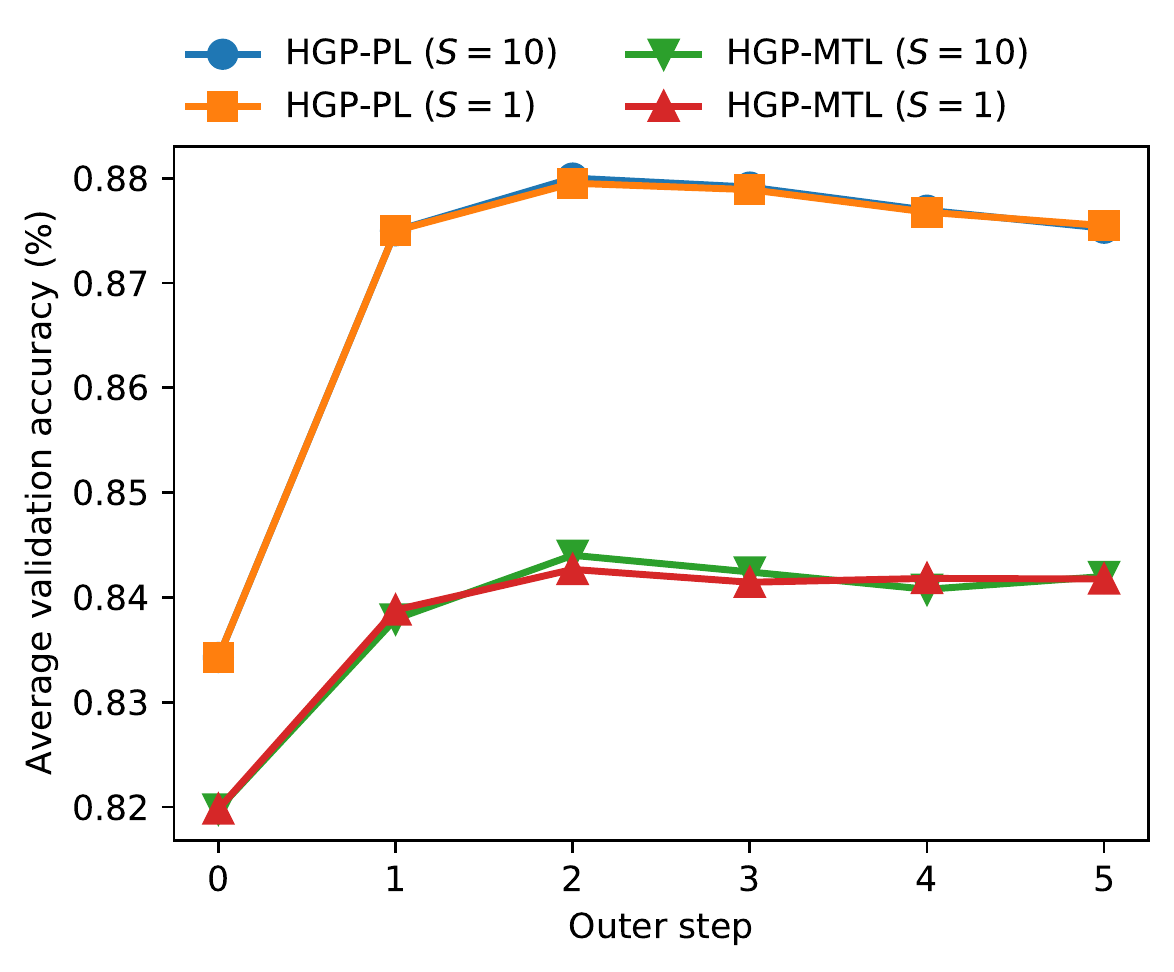}
        \caption{Static Undirected (\texttt{FixU}) }
        \label{fig:fixu}
    \end{subfigure}
    
    \vspace{10pt} 
    \begin{subfigure}[b]{0.49\textwidth}
        \includegraphics[width=\textwidth]{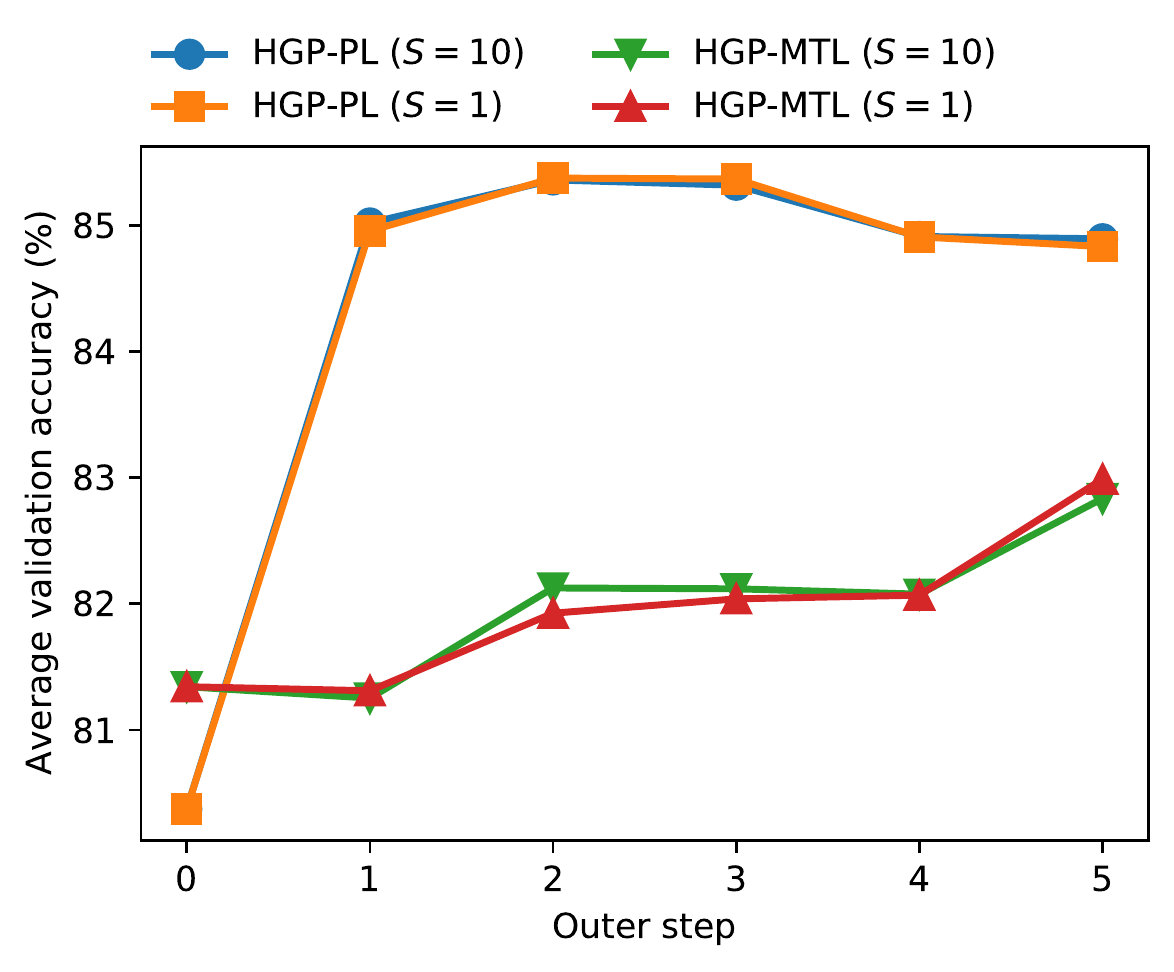}
        \caption{Time-Varying Undirected (\texttt{RandU})}
        \label{fig:randu}
    \end{subfigure}
    \hfill
    \begin{subfigure}[b]{0.49\textwidth}
        \includegraphics[width=\textwidth]{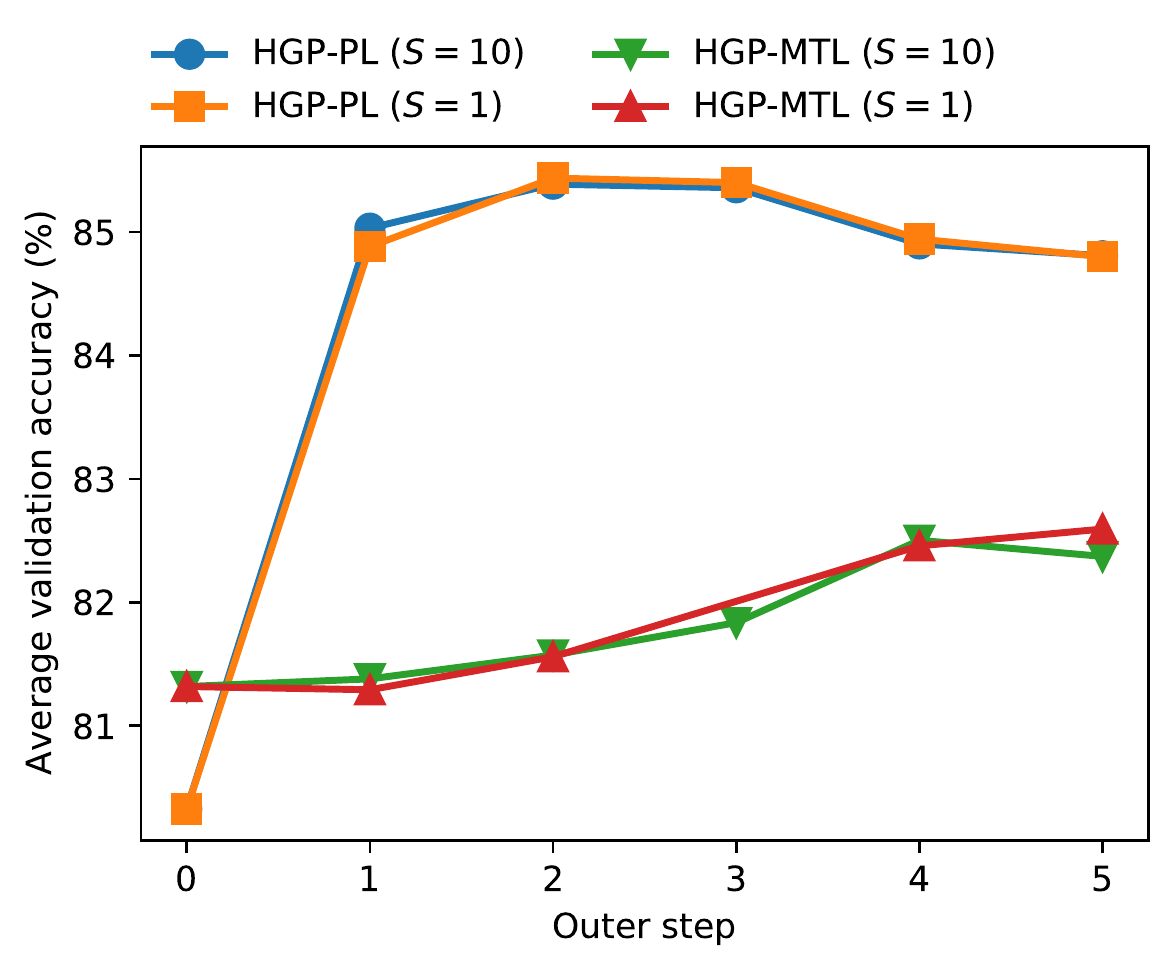}
        \caption{Time-Varying Directed (\texttt{RandD})}
        \label{fig:randd}
    \end{subfigure}
    \caption{Validation accuracies of outer-steps over different communication networks.}
    \label{fig:val_accs}
\end{figure}

\newpage
\section{Proof of \cref{lem:proximal} }

From the optimality condition of \cref{eq:fl}, we will show,
\begin{equation} \label{eq:lemma_alt}
\boldsymbol{x} ( \boldsymbol{\lambda} )=\boldsymbol{\Theta }(\boldsymbol{x} -\eta \partial _{\boldsymbol{x}} g(\boldsymbol{x} ,\boldsymbol{\lambda })) \ \Leftrightarrow \ \partial _{\boldsymbol{x}} g(\boldsymbol{x} (\boldsymbol{\lambda } ),\boldsymbol{\lambda } )=0\ \ \text{s.t.}\ ~\boldsymbol{x}_{i} =\boldsymbol{x}_{j} ,\forall i,j.
\end{equation}
We first show from right to left. Given that $\displaystyle \boldsymbol{x}_{i} =\boldsymbol{x}_{j} ,\forall i,j\ $ and $\sum _{k} \partial _{\boldsymbol{x}_{k}} g_{k} (\boldsymbol{x}_{k} (\boldsymbol{\lambda } ),\boldsymbol{\lambda }_{k} )=0$, the following relation holds,
\begin{equation*}
\frac{1}{n}\sum _{k}\left\{\boldsymbol{x}_{k}(\boldsymbol{\lambda }) -\eta\partial _{\boldsymbol{x}_{k}} g_{k} \left(\boldsymbol{x}_{k} (\boldsymbol{\lambda } ),\boldsymbol{\lambda }_{k}\right)\right\} =\boldsymbol{x}_{i}\left(\boldsymbol{\lambda } \right),\ \forall i.
\end{equation*}
Therefore, 
\begin{align*}
\boldsymbol{\Theta }(\boldsymbol{x} -\eta \partial _{\boldsymbol{x}} g(\boldsymbol{x} ,\boldsymbol{\lambda })) & =\frac{1}{n}\begin{bmatrix}
\boldsymbol{I} & \cdots  & \boldsymbol{I} \\
\vdots  & \ddots  & \vdots \\
\boldsymbol{I} & \cdots  & \boldsymbol{I}
\end{bmatrix}\left[\begin{array}{ c }
\boldsymbol{x}_{1}(\boldsymbol{\lambda }) -\eta\partial_{\boldsymbol{x}_{1}} g_{1} (\boldsymbol{x}_{1}(\boldsymbol{\lambda }) ,\boldsymbol{\lambda }_{1} )\\
\vdots \\
\boldsymbol{x}_{n}(\boldsymbol{\lambda }) -\eta \partial_{ \boldsymbol{x}_{n}} g_{n} (\boldsymbol{x}_{n}(\boldsymbol{\lambda }) ,\boldsymbol{\lambda }_{n} )
\end{array}\right]\\
 & =\left[\begin{array}{ c }
\frac{1}{n}\sum _{k}\left\{\boldsymbol{x}_{k}(\boldsymbol{\lambda }) -\eta \partial_{ \boldsymbol{x}_{k}} g_{k} (\boldsymbol{x}_{k} (\boldsymbol{\lambda } ),\boldsymbol{\lambda }_{k} )\right\}\\
\vdots \\
\frac{1}{n}\sum _{k}\left\{\boldsymbol{x}_{k}(\boldsymbol{\lambda }) -\eta \partial_{ \boldsymbol{x}_{k}} g_{k} (\boldsymbol{x}_{k} (\boldsymbol{\lambda } ),\boldsymbol{\lambda }_{k} )\right\}
\end{array}\right]\\
\  & =\left[\begin{array}{ c }
\boldsymbol{x}_{1}(\boldsymbol{\lambda })\\
\vdots \\
\boldsymbol{x}_{n}(\boldsymbol{\lambda })
\end{array}\right] = \boldsymbol{x} ( \boldsymbol{\lambda} ),
\end{align*}
confirming that the left statement is true. 

Next, we demonstrate from left to right. We will first show that $\displaystyle \boldsymbol{x}_{i} =\boldsymbol{x}_{j}$ for every $i,j$. Let $-\eta \sum _{k}\partial_{ \boldsymbol{x}_{k}} g_{k} (\boldsymbol{x}_{k} (\boldsymbol{\lambda } ),\boldsymbol{\lambda }_{k} ) =\nabla g$. Then the matrix expression of $\boldsymbol{x} ( \boldsymbol{\lambda} )=\boldsymbol{\Theta }(\boldsymbol{x} -\eta \partial _{\boldsymbol{x}} g(\boldsymbol{x} ,\boldsymbol{\lambda }))$ is 
\begin{align} \label{eq:mat_left}
\left[\begin{array}{ c }
\boldsymbol{x}_{1} (\boldsymbol{\lambda } )\\
\vdots \\
\boldsymbol{x}_{n} (\boldsymbol{\lambda } )
\end{array}\right] & =\left[\begin{array}{ c }
\frac{1}{n}\sum _{k}\boldsymbol{x}_{k} (\boldsymbol{\lambda } ) + \nabla g\\
\vdots \\
\frac{1}{n}\sum _{k}\boldsymbol{x}_{k} (\boldsymbol{\lambda } ) + \nabla g
\end{array}\right] .
\end{align}
We therefore have 
\begin{align*}
\nabla g = \boldsymbol{x}_{i} (\boldsymbol{\lambda } )-\frac{1}{n}\sum _{k}\boldsymbol{x}_{k} (\boldsymbol{\lambda } ) & =\boldsymbol{x}_{j} (\boldsymbol{\lambda } )-\frac{1}{n}\sum _{k}\boldsymbol{x}_{k} (\boldsymbol{\lambda } ),\ \forall i,j.
\end{align*}
Thus, the consensus condition $
\boldsymbol{x}_{i} (\boldsymbol{\lambda } ) =\boldsymbol{x}_{j} (\boldsymbol{\lambda } ),\ \forall i,j
$ holds true. 
Because of this consensus condition and \cref{eq:mat_left}, we have $\displaystyle \nabla g = \boldsymbol{x}_{i}(\boldsymbol{\lambda }) -\frac{1}{n}\sum _{k}\boldsymbol{x}_{k}(\boldsymbol{\lambda }) =0,\ \forall i$ . Recalling $ \eta \neq 0$, we obtain $\sum _{k} \partial _{\boldsymbol{x}_{k}} g_{k} (\boldsymbol{x}_{k} (\boldsymbol{\lambda } ),\boldsymbol{\lambda }_{k} )=0$ , confirming that the right statement of \cref{eq:lemma_alt} is true. \qed 

\section{Proof of \cref{thm:error} }

In this section, we provide the proof of \cref{thm:error}.
To begin with, we restate the assumptions and theorem.

\setcounter{assumption}{0}
\setcounter{theorem}{0}

\begin{assumption}[$B$-strong connectivity] \label{ass:connect}
The graph with edge set $\bigcup _{s=tB}^{(t+1)B-1}\mathcal{E}(s)$ is strongly-connected for every $t\geq 0$.
\end{assumption}

\begin{assumption}[Strong convexity] \label{ass:convex}
    For every $i=1,\ldots,n$, $g_{i}$ is strongly convex with respect to the first argument.
\end{assumption}

\begin{assumption}[Jacobian oracle] \label{ass:jacob}
Let $\xi=\{\xi_1,\ldots,\xi_n\}$  and $\hat{g} (\boldsymbol{x},\boldsymbol{\lambda}; \xi)=\sum_{k=1}^{n} g_{k} (\boldsymbol{x}_{k} ,\boldsymbol{\lambda} _{k} ;\xi _{k} ) $.
\begin{itemize}
    \item Unbiasedness \: $\mathbb{E}_\xi [\partial_{\bm x}^2 \hat{g}] = \partial_{\bm x}^2 g$, \: $\mathbb{E}_\xi [\partial_{\bm{x\lambda}}^2 \hat{g}] = \partial_{\bm{x\lambda}}^2 g$
    \item Boundedness \: $\exists \kappa_{\bm x}, \kappa_{\bm \lambda} > 0$, $\forall \bm{x}, \bm{\lambda}, \xi$, 
    
    \hspace{64pt} $\sigma_{\max}( \partial_{\boldsymbol{x}}^2\hat{g} - \partial _{\boldsymbol{x}}^2 g ) \leq \kappa_{\boldsymbol{x}}$, \: $\sigma_{\max}( \partial_{\boldsymbol{x}\boldsymbol{\lambda}}^2\hat{g} - \partial _{\boldsymbol{x}\boldsymbol{\lambda}}^2 g ) \leq \kappa_{\boldsymbol{\lambda}}$
\end{itemize}
\end{assumption}
\begin{theorem} \label{thm:error}
    Let $\alpha= \min_{{\boldsymbol{\lambda} ,\xi}}\left\{\sigma_{\min}( \partial_{\boldsymbol{x}}^2\hat{g})\right\}$, $\beta= \max_{{\boldsymbol{\lambda} ,\xi}}\left\{\sigma_{\max}( \partial_{\boldsymbol{x}\boldsymbol{\lambda}}^2\hat{g})\right\}$.
    Under \crefrange{ass:connect}{ass:jacob} and $0 < \alpha < \nicefrac{1}{\eta}$, there exists constants $0\leq \tau < 1$, $\delta >0$, and $S \geq \nicefrac{\log \frac{\delta}{16\sqrt{n}} \frac{\eta \alpha}{1-\eta \alpha }}{\log \tau}$,such that with probability at least $1-\epsilon $,
\begin{align*}
    \frac{\left\Vert \boldsymbol{v}^{(M)} -\mathrm{d}_{\boldsymbol{\lambda}} f \right\Vert }{\left\Vert \partial _{\boldsymbol{x}}  f   \right\Vert } & \leq \frac{4 \sqrt{\kappa_{\bm \lambda}^2 + \kappa_{\bm x}^2 \frac{\beta^2}{\alpha^2}}}{\alpha} \left(1 + \frac{8 \sqrt{n}}{\delta}\tau^S\right) \sqrt{2 \log \frac{n (d_x + d_\lambda)}{\epsilon}} \\
    & \quad + \frac{2\beta}{\eta \alpha^2} \frac{8 \sqrt{n}}{\delta} \tau^S +\exp (-O(M)) ,
\end{align*}
where $\exp (-O(M))$ denotes the exponentially diminishing term.
\end{theorem}

\subsection{Preliminary}

Let $\hat{\boldsymbol{\Theta}}$ be the $S$-step Push-Sum operator (\cref{alg:pushsum}) such that $\left[ y_i^{(S)} \right]_{i=1}^n = \hat{\boldsymbol{\Theta}} \left[ y_i^{(0)} \right]_{i=1}^n$.

\begin{lemma}
    \label{lem:lipschitz}
    There exists $\delta > 0$ and $0 < \tau < 1$ such that
    \begin{align}
        \sigma_{\max}\left( \boldsymbol{\Theta} - \hat{\boldsymbol{\Theta}} \right) \le \frac{8 \sqrt{n}}{\delta} \tau^S .
    \end{align}
\end{lemma}

\begin{proof}
From Lemma 1 of \cite{nedic2014distributed} and \cref{ass:connect}, there exists $\delta > 0$ and $0 < \tau < 1$ such that
\begin{align*}
    \left\| \boldsymbol{\Theta} y - \hat{\boldsymbol{\Theta}} y \right\|_\infty \le \frac{8 \|y\|_1}{\delta} \tau^S, \: \forall y \in \mathbb{R}^n .
\end{align*}
We can conclude the claim by recalling
\begin{align*}
    \left\| A y \right\|_\infty \le \sigma_{\max}\left( A \right) \|y\| , \: \forall A \in \mathbb{R}^{n \times n} , \quad \text{and} \quad \|y\|_1 \le \sqrt{n} \|y\| .
\end{align*}
\end{proof}

\begin{remark}[Explicit formula of $\bm{v}^{(m)}$] \label{rem:vm}
    $\bm{v}^{(M)}$ in \cref{alg:hgp} is expressed as
    \begin{align}
        \bm{v}^{(M)} = - \eta \sum_{m=0}^{M-1} \partial_{\bm{x \lambda}}^2 \hat{g}^{(m)} \hat{\boldsymbol{\Theta}} \prod_{m'=0}^{m-1} \left( \left( I - \eta \partial_{\bm x}^2 \hat{g}^{(m')} \right) \hat{\boldsymbol{\Theta}} \right) \partial_{\bm x} f +\partial _{\boldsymbol{\lambda}} f ,
    \end{align}
    where
    \begin{align*}
        \partial_{\bm{x \lambda}}^2 \hat{g}^{(m)} &= \mathrm{diag}\left(\left[\partial _{\boldsymbol{x}_i\boldsymbol{\lambda}_i} ^2 g_i(\boldsymbol{x}_i(\boldsymbol{\lambda}), \boldsymbol{\lambda}_i; \hat{\xi}_i^{(m)})\right]_{i=1}^n\right) \in \mathbb{R}^{n d_{\bm \lambda} \times n d_{\bm{x}}} , \\
        \partial_{\bm{x}}^2 \hat{g}^{(m)} &= \mathrm{diag}\left(\left[\partial _{\boldsymbol{x}_i} ^2 g_i(\boldsymbol{x}_i(\boldsymbol{\lambda}), \boldsymbol{\lambda}_i; \hat{\xi}_i^{(m)})\right]_{i=1}^n\right) \in \mathbb{R}^{n d_{\bm x} \times n d_{\bm{x}}} ,
    \end{align*}
    are block-diagonal matrices, and $\prod_{m=0}^{M-1} A^{(m)} = A^{(M-1)} A^{(M-2)} \ldots A^{(0)}$ for matrices $A^{(0)}, A^{(1)}, \ldots, A^{(M-1)}$. We also define $\prod_{m=0}^{-1} = I$ for convention.
\end{remark}

\subsection{Proof of \cref{thm:error}}

Recall that the true hyper-gradient $\mathrm{d}_{\boldsymbol{\lambda}} f$ is given as
\begin{equation*}
\mathrm{d}_{\boldsymbol{\lambda}} f =-\eta \partial _{\boldsymbol{x}\boldsymbol{\lambda}}^{2} g   \boldsymbol{\Theta} \sum\limits _{m=0}^{\infty }\left(\left( \boldsymbol{I}-\eta \partial _{\boldsymbol{x}}^{2} g   \right) \boldsymbol{\Theta} \right)^{m} \partial _{\boldsymbol{x}} f  +\partial _{\boldsymbol{\lambda}} f .
\end{equation*}
From \cref{rem:vm}, we have
\begin{align*}
    v^{(M)} - \mathrm{d}_{\boldsymbol{\lambda}} f =& - \underbrace{\eta \sum_{m=0}^{M-1} \left\{ \partial_{\bm{x \lambda}}^2 \hat{g}^{(m)} \hat{\boldsymbol{\Theta}} \prod_{m'=0}^{m-1} \left( \left( I - \eta \partial_{\bm x}^2 \hat{g}^{(m')} \right) \hat{\boldsymbol{\Theta}} \right) \partial_{\bm x} f -  \partial _{\boldsymbol{x}\boldsymbol{\lambda}}^{2} g   \boldsymbol{\Theta} \left(\left( \boldsymbol{I}-\eta \partial _{\boldsymbol{x}}^{2} g   \right) \boldsymbol{\Theta} \right)^{m} \partial _{\boldsymbol{x}} f \right\}}_{\text{(A)}} \\
    & +\underbrace{ \eta \partial _{\boldsymbol{x}\boldsymbol{\lambda}}^{2} g   \boldsymbol{\Theta} \sum\limits _{m=M}^{\infty }\left(\left( \boldsymbol{I}-\eta \partial _{\boldsymbol{x}}^{2} g   \right) \boldsymbol{\Theta} \right)^{m} \partial _{\boldsymbol{x}} f }_{\text{(B)}}.
\end{align*}

We can bound term (B) as
\begin{align*}
    \left\| \text{(B)} \right\| & \le \eta \beta \sum_{m=M+1}^\infty (1 - \eta \alpha)^m \|\partial _{\boldsymbol{x}} f\| = \eta \beta \frac{(1 - \eta \alpha)^{M}}{\eta \alpha} \|\partial _{\boldsymbol{x}} f\| = \exp\left(-O(M)\right) .
\end{align*}

To bound the term (A), we decompose the terms as
\begin{align*}
    \text{(A)} & = \underbrace{\eta \sum_{m=0}^{M-1} \left(\partial_{\bm{x \lambda}}^2 \hat{g}^{(m)} - \partial_{\bm{x \lambda}}^2 g\right) \hat{\boldsymbol{\Theta}} \prod_{m'=0}^{m-1} \left( \left( I - \eta \partial_{\bm x}^2 \hat{g}^{(m')} \right) \hat{\boldsymbol{\Theta}} \right) \partial _{\boldsymbol{x}} f}_{\text{(A1)}} \\
    & \quad + \underbrace{\eta \sum_{m=0}^{M-1} \partial_{\bm{x \lambda}}^2 g \left(\hat{\boldsymbol{\Theta}} - \boldsymbol{\Theta}\right) \prod_{m'=0}^{m-1} \left( \left( I - \eta \partial_{\bm x}^2 \hat{g}^{(m')} \right) \hat{\boldsymbol{\Theta}} \right) \partial _{\boldsymbol{x}} f}_{\text{(A2)}} \\
    & \quad - \underbrace{\eta \sum_{m=1}^{M-1} \partial_{\bm{x \lambda}}^2 g \boldsymbol{\Theta} \sum_{t=0}^{m-1} \left(\left( \boldsymbol{I}-\eta \partial _{\boldsymbol{x}}^{2} g   \right) \boldsymbol{\Theta} \right)^{t} \eta \left(\partial_{\bm{x}}^2 \hat{g}^{(m-t-1)} - \partial_{\bm{x}}^2 g\right) \hat{\boldsymbol{\Theta}} \prod_{m'=0}^{m-t-2} \left( \left( I - \eta \partial_{\bm x}^2 \hat{g}^{(m')} \right) \hat{\boldsymbol{\Theta}} \right) \partial _{\boldsymbol{x}} f}_{\text{(A3)}} \\
    & \quad + \underbrace{\eta \sum_{m=1}^{M-1} \partial_{\bm{x \lambda}}^2 g \boldsymbol{\Theta} \sum_{t=0}^{m-1} \left(\left( \boldsymbol{I}-\eta \partial _{\boldsymbol{x}}^{2} g   \right) \boldsymbol{\Theta} \right)^{t} \left( \boldsymbol{I}-\eta \partial _{\boldsymbol{x}}^{2} g   \right) \left( \hat{\boldsymbol{\Theta}} - \boldsymbol{\Theta} \right) \prod_{m'=0}^{m-t-2} \left( \left( I - \eta \partial_{\bm x}^2 \hat{g}^{(m')} \right) \hat{\boldsymbol{\Theta}} \right) \partial _{\boldsymbol{x}} f}_{\text{(A4)}} .
\end{align*}
We can thus bound (A) as
\begin{align*}
    \| \text{(A)} \| \le \|\text{(A1)} - \text{(A3)}\| + \|\text{(A2)} + \text{(A4)}\|.
\end{align*}

Below, we bound each term.

\paragraph{Bounding $\|\text{(A2)} + \text{(A4)}\|$}
The assumptions $0 \le \alpha \le \nicefrac{1}{\eta}$ and $S \geq \nicefrac{\log \frac{\delta}{16\sqrt{n}} \frac{\eta \alpha}{1-\eta \alpha }}{\log \tau}$ ensures
\begin{align}
    & 1 + \frac{8 \sqrt{n}}{\delta} \tau^S \le 1 + \frac{8 \sqrt{n}}{\delta} \frac{\delta}{16\sqrt{n}} \frac{\eta \alpha}{1-\eta \alpha } = 1 + \frac{1}{2} \frac{\eta \alpha}{1-\eta \alpha } = \frac{1}{2}\left( 1 + \frac{1}{1 - \eta \alpha} \right) , \label{eq:S1} \\
    & \left(1 - \eta \alpha\right) \left(1 + \frac{8 \sqrt{n}}{\delta} \tau^S\right) \le (1 - \eta \alpha) \frac{1}{2}\left( 1 + \frac{1}{1 - \eta \alpha} \right) = 1 - \frac{\eta \alpha}{2} < 1 . \label{eq:S2}
\end{align}

By \cref{lem:lipschitz}, we have
\begin{align}
    & \| \text{(A2)} + \text{(A4)} \| \nonumber \\
    & \le \eta \sum_{m=0}^{M-1} \beta \frac{8 \sqrt{n}}{\delta} \tau^S \left(\left(1 - \eta \alpha\right) \left(1 + \frac{8 \sqrt{n}}{\delta} \tau^S\right)\right)^{m} \|\partial_{\bm x} f\| \nonumber \\
    & \quad + \eta \sum_{m=1}^{M-1} \beta \sum_{t=0}^{m-1} (1 - \eta \alpha)^{t+1} \frac{8 \sqrt{n}}{\delta} \tau^S \left(\left(1 - \eta \alpha\right) \left(1 + \frac{8 \sqrt{n}}{\delta} \tau^S\right)\right)^{m-t-1} \|\partial_{\bm x} f\| \nonumber \\
    & = \|\partial_{\bm x} f\| \eta \beta \frac{8 \sqrt{n}}{\delta} \tau^S \sum_{m=0}^{M-1} \left(\left(1 - \eta \alpha\right) \left(1 + \frac{8 \sqrt{n}}{\delta} \tau^S\right)\right)^{m} \nonumber \\
    & \quad + \|\partial_{\bm x} f\| \eta \beta \frac{8 \sqrt{n}}{\delta} \tau^S \sum_{m=1}^{M-1} (1 - \eta \alpha)^{m} \sum_{t=0}^{m-1} \left(1 + \frac{8 \sqrt{n}}{\delta} \tau^S\right)^{m-t-1}  \nonumber \\
    & = \|\partial_{\bm x} f\| \eta \beta \frac{8 \sqrt{n}}{\delta} \tau^S \sum_{m=0}^{M-1} \left(\left(1 - \eta \alpha\right) \left(1 + \frac{8 \sqrt{n}}{\delta} \tau^S\right)\right)^{m} \nonumber \\
    & \quad + \|\partial_{\bm x} f\| \eta \beta \frac{8 \sqrt{n}}{\delta} \tau^S \sum_{m=1}^{M-1} (1 - \eta \alpha)^{m} \frac{\left(1 + \frac{8 \sqrt{n}}{\delta} \tau^S\right)^m - 1}{\frac{8 \sqrt{n}}{\delta} \tau^S}  \nonumber \\
    & = \|\partial_{\bm x} f\| \eta \beta \frac{8 \sqrt{n}}{\delta} \tau^S \frac{1}{1 - \left(1 - \eta \alpha\right) \left(1 + \frac{8 \sqrt{n}}{\delta} \tau^S\right)} \nonumber \\
    & \quad + \|\partial_{\bm x} f\| \eta \beta \left(1 - \eta \alpha\right) \left\{\frac{1 + \frac{8 \sqrt{n}}{\delta} \tau^S}{1 - \left(1 - \eta \alpha\right) \left(1 + \frac{8 \sqrt{n}}{\delta} \tau^S\right)} - \frac{1}{1 - (1 - \eta \alpha)}\right\} + \exp\left( - O(M) \right) \nonumber \\
    & = \|\partial_{\bm x} f\| \eta \beta \frac{8 \sqrt{n}}{\delta} \tau^S \frac{1}{1 - \left(1 - \eta \alpha\right) \left(1 + \frac{8 \sqrt{n}}{\delta} \tau^S\right)} \nonumber \\
    & \quad + \|\partial_{\bm x} f\| \frac{\beta}{\alpha} \left(1 - \eta \alpha\right) \frac{8 \sqrt{n}}{\delta} \tau^S \frac{1}{1 - \left(1 - \eta \alpha\right) \left(1 + \frac{8 \sqrt{n}}{\delta} \tau^S\right)} + \exp\left( - O(M) \right) \nonumber \\
    & \le \|\partial_{\bm x} f\| \frac{2\beta}{\alpha} \frac{8 \sqrt{n}}{\delta} \tau^S + \|\partial_{\bm x} f\| \frac{2\beta}{\alpha} \frac{1 - \eta \alpha}{\eta \alpha} \frac{8 \sqrt{n}}{\delta} \tau^S + \exp\left( - O(M) \right) \nonumber \\
    & = \|\partial_{\bm x} f\| \frac{2\beta}{\eta \alpha^2} \frac{8 \sqrt{n}}{\delta} \tau^S + \exp\left( - O(M) \right) , \label{eq:a24} 
\end{align}
where we used \cref{eq:S2} for the last inequality.

\paragraph{Bounding $\|\text{(A1)} - \text{(A3)}\|$}
We use Matrix Azuma's inequality to bound $\|\text{(A1)} - \text{(A3)}\|$.
Recall that we can express $\text{(A1)} - \text{(A3)}$ as 
\begin{align} \label{eq:A13}
    \text{(A1)} - \text{(A3)} = \left(\sum_{m=0}^{M-1} X^{(m)} + \sum_{m=0}^{M-2} Y^{(m)} \right) \partial_{\bm x} f
\end{align}
where
\begin{align*}
    X^{(m)} &= \eta \left(\partial_{\bm{x \lambda}}^2 \hat{g}^{(m)} - \partial_{\bm{x \lambda}}^2 g\right) \hat{\boldsymbol{\Theta}} \prod_{m'=0}^{m-1} \left( \left( I - \eta \partial_{\bm x}^2 \hat{g}^{(m')} \right) \hat{\boldsymbol{\Theta}} \right) , \\
    Y^{(m)} &= - \eta^2 \sum_{t=m+1}^{M-2} \partial_{\bm{x \lambda}}^2 g \boldsymbol{\Theta} \left(\left( \boldsymbol{I}-\eta \partial _{\boldsymbol{x}}^{2} g   \right) \boldsymbol{\Theta} \right)^{t-m-1} \left(\partial_{\bm{x}}^2 \hat{g}^{(m)} - \partial_{\bm{x}}^2 g\right) \hat{\boldsymbol{\Theta}} \prod_{m'=0}^{m-1} \left( \left( I - \eta \partial_{\bm x}^2 \hat{g}^{(m')} \right) \hat{\boldsymbol{\Theta}} \right) .
\end{align*}
By the independence of $\hat{\xi}_i^{(m)}$ and $\hat{\hat{\xi}}_i^{(m)}$ in \cref{alg:hgp}, we have 
\begin{align*}
    & \mathbb{E}\left[X^{(m)} \mid X^{(0)}, \ldots, X^{(m-1)}, Y^{(0)}, \ldots, Y^{(m-1)}\right] = 0 , \\
    & \mathbb{E}\left[Y^{(m)} \mid X^{(0)}, \ldots, X^{(m)}, Y^{(0)}, \ldots, Y^{(m-1)}\right] = 0 ,
\end{align*}
from the fact that $\mathbb{E}[\partial_{\bm{x \lambda}}^2 \hat{g}^{(m)}] = \partial_{\bm{x \lambda}}^2 g$ and $\mathbb{E}[\partial_{\bm{x}}^2 \hat{g}^{(m)}] = \partial_{\bm{x}}^2 g$.
Thus, we can apply Matrix Azuma's inequality to \cref{eq:A13} and obtain, with probability at least $1 - \epsilon$, 
\begin{align*}
    \| \text{(A1)} - \text{(A3)} \| \le \|\partial_{\bm x} f\| \sqrt{8 s^2 \log \frac{n (d_x + d_\lambda)}{\epsilon}} ,
\end{align*}
where
\begin{align*}
    & s^2 \le \left(\sum_{m=0}^{M-1} \max \sigma_{\max}\left( X^{(m)} {X^{(m)}}^\top \right) + \sum_{m=0}^{M-2} \max \sigma_{\max}\left( Y^{(m)} {Y^{(m)}}^\top \right) \right) .
\end{align*}
By \cref{lem:lipschitz} and \cref{ass:jacob}, we can see that
\begin{align*}
    \sigma_{\max}\left( X^{(m)} {X^{(m)}}^\top \right) \le \left\{ \eta \kappa_{\bm \lambda} \left(1 - \eta \alpha\right)^m \left(1 + \frac{8 \sqrt{n}}{\delta}\tau^S\right)^{m+1} \right\}^2 , 
\end{align*}
and
\begin{align*}
    \sigma_{\max}\left( Y^{(m)} {Y^{(m)}}^\top \right) & \le \left\{ \eta^2 \beta \sum_{t=m+1}^{M-2} (1 - \eta \alpha)^{t-m-1} \kappa_{\bm x} \left(1 - \eta \alpha\right)^m \left(1 + \frac{8 \sqrt{n}}{\delta}\tau^S\right)^{m+1} \right\}^2 \\
    & = \left\{\kappa_{\bm x} \frac{\eta \beta}{\alpha} \left(1 - \eta \alpha\right)^m \left(1 + \frac{8 \sqrt{n}}{\delta}\tau^S\right)^{m+1} \right\}^2 + \exp\left(-O(M)\right) .
\end{align*}
Hence, we have
\begin{align*}
    s^2 & \le \eta^2 \kappa_{\bm \lambda}^2 \sum_{m=0}^{M-1} \left(1 - \eta \alpha\right)^{2m} \left(1 + \frac{8 \sqrt{n}}{\delta}\tau^S\right)^{2m+2} \\
    & \quad + \kappa_{\bm x}^2 \frac{\eta^2 \beta^2}{\alpha^2} \sum_{m=0}^{M-2} \left(1 - \eta \alpha\right)^{2m} \left(1 + \frac{8 \sqrt{n}}{\delta}\tau^S\right)^{2m+2} + \exp\left(-O(M)\right) \\
    &= \eta^2 \left(\kappa_{\bm \lambda}^2 + \kappa_{\bm x}^2 \frac{\beta^2}{\alpha^2} \right) \frac{1}{1 - \left(1 - \eta \alpha\right)^2 \left(1 + \frac{8 \sqrt{n}}{\delta}\tau^S\right)^2} \left(1 + \frac{8 \sqrt{n}}{\delta}\tau^S\right)^2 + \exp\left(-O(M)\right) .
\end{align*}
Here, by \cref{eq:S1}, we have
\begin{align*}
    1 - (1 - \eta \alpha)^2 \left(1 + \frac{8 \sqrt{n}}{\delta} \tau^S \right)^2 \ge 1 - \left(1 - \frac{\eta \alpha}{2} \right)^2 \ge \frac{\eta^2 \alpha^2}{4} .
\end{align*}
We can thus conclude that
\begin{align}
    \| \text{(A1)} - \text{(A3)} \| \le \|\partial_{\bm x} f\| \frac{4 \sqrt{\kappa_{\bm \lambda}^2 + \kappa_{\bm x}^2 \frac{\beta^2}{\alpha^2}}}{\alpha} \left(1 + \frac{8 \sqrt{n}}{\delta}\tau^S\right) \sqrt{2 \log \frac{n (d_x + d_\lambda)}{\epsilon}} + \exp\left(-O(M)\right)  . \label{eq:a13}
\end{align}

Finally, \cref{thm:error} follows from \cref{eq:a24} and \cref{eq:a13}.

\section{Extensions of \cref{thm:error}}

Here, we present a few extensions of \cref{thm:error}.

\subsection{Use of Exact Jacobians in \cref{alg:hgp}}
In \cref{alg:hgp}, we assume that we only have access to the unbiased estimates of Jacobians $\partial_{\bm{x\lambda}}^2 \hat{g}$ and $\partial_{\bm{x}}^2 \hat{g}$.
If we use the true Jacobians $\partial_{\bm{x\lambda}}^2 g$ and $\partial_{\bm{x}}^2 g$, the terms (A1) and (A3) gets zeros in the proof of \cref{thm:error}.
Thus, only the term $\|\text{(A2)} + \text{(A4)}\|$ remains.
\begin{corollary}\label{cor:true_jacob}
    Suppose that we run \cref{alg:hgp} using the true Jacobians $\partial_{\bm{x\lambda}}^2 g$ and $\partial_{\bm{x}}^2 g$.
    Then, under the same assumptions as \cref{thm:error}, we have
    \begin{align*}
        \frac{\left\Vert \boldsymbol{v}^{(M)} -\mathrm{d}_{\boldsymbol{\lambda}} f \right\Vert }{\left\Vert \partial _{\boldsymbol{x}}  f   \right\Vert } \le \frac{2\beta}{\eta \alpha^2} \frac{8 \sqrt{n}}{\delta} \tau^S +\exp (-O(M)) .
    \end{align*}
\end{corollary}

\subsection{Use of Mini-Batch Estimates of Jacobians in \cref{alg:hgp}}
One popular choice of unbiased Jacobians are the mini-batch estimates.
If we use mini-batch of size $b$, we have $\kappa_{\bm \lambda}, \kappa_{\bm x}$ in \cref{ass:jacob} as
\begin{align*}
    \kappa_{\bm \lambda} = O_P\left(\frac{1}{\sqrt{b}}\right), \quad \kappa_{\bm x} = O_P\left(\frac{1}{\sqrt{b}}\right) .
\end{align*}
Then, we have the error bound as follows.
\begin{corollary}\label{cor:mini_batch_jacob}
    Suppose that we run \cref{alg:hgp} use the mini-batch of size $b$ to estimate Jacobians $\partial_{\bm{x\lambda}}^2 \hat{g}$ and $\partial_{\bm{x}}^2 \hat{g}$.
    Then, under the same assumptions as \cref{thm:error}, there exists a constant $C > 0$, and we have
    \begin{align*}
        \frac{\left\Vert \boldsymbol{v}^{(M)} -\mathrm{d}_{\boldsymbol{\lambda}} f \right\Vert }{\left\Vert \partial _{\boldsymbol{x}}  f   \right\Vert } & \leq \frac{4 C \sqrt{1 + \frac{\beta^2}{\alpha^2}}}{\alpha} \left(1 + \frac{8 \sqrt{n}}{\delta}\tau^S\right) \sqrt{\frac{2}{b} \log \frac{n (d_x + d_\lambda)}{\epsilon}} \\
        & \quad + \frac{2\beta}{\eta \alpha^2} \frac{8 \sqrt{n}}{\delta} \tau^S +\exp (-O(M)) .
    \end{align*}
\end{corollary}
Apparently, \cref{cor:true_jacob} is recovered as a special case of \cref{cor:mini_batch_jacob} when $b \to \infty$.

\subsection{Use of Inexact Inner-Solution}
In the derivation of \cref{alg:hgp}, we assume that the exact inner-solution $\bm{x}(\bm{\lambda})$ in \cref{eq:bilevel} is obtained.
In practice, this is not always true because we typically stop the optimization in the middle of the training once the solution gets sufficiently close to the optimum.
Suppose $\hat{\bm{x}}(\bm{\lambda})$ be the inexact solution close to $\bm{x}(\bm{\lambda})$.
The question is how the error $\hat{\bm{x}}(\bm{\lambda}) - \bm{x}(\bm{\lambda})$ of the inner-solution affects the quality of the hyper-gradient estimated by \cref{alg:hgp}.

Here, we adopt the next assumption.
\begin{assumption}\label{ass:glip}
    There exists $\nu_f, \nu_g > 0$ such that, $\forall \bm{x}, \bm{y}$, 
    \begin{align*}
        \left\| \partial_{\bm{\lambda}} f(\bm{x}, \bm{\lambda}) - \partial_{\bm{\lambda}} f(\bm{y}, \bm{\lambda}) \right\| & \le \nu_f \|\bm{x} - \bm{y}\|, \\
        \left\| \partial_{\bm{x}} f(\bm{x}, \bm{\lambda}) - \partial_{\bm{x}} f(\bm{y}, \bm{\lambda}) \right\| & \le \nu_f \|\bm{x} - \bm{y}\|, \\
        \sigma_{\max} \left( \partial_{\bm{x \lambda}}^2 g(\bm{x}, \bm{\lambda}) - \partial_{\bm{x \lambda}}^2 g(\bm{y}, \bm{\lambda}) \right) & \le \nu_g \|\bm{x} - \bm{y}\|, \\
        \sigma_{\max} \left( \partial_{\bm{x}}^2 g(\bm{x}, \bm{\lambda}) - \partial_{\bm{x}}^2 g(\bm{y}, \bm{\lambda}) \right) & \le \nu_g \|\bm{x} - \bm{y}\| .
    \end{align*}
\end{assumption}

We then obtain the generalization of \cref{thm:error} as follows.
\begin{theorem} \label{thm:error2}
    Suppose we run \cref{alg:hgp} using the inexact inner-solution $\hat{\bm{x}}(\bm{\lambda})$.
    Under \crefrange{ass:connect}{ass:glip} and $0 < \alpha < \nicefrac{1}{\eta}$, there exists constants $0\leq \tau < 1$, $\delta >0$, and $S \geq \nicefrac{\log \frac{\delta}{16\sqrt{n}} \frac{\eta \alpha}{1-\eta \alpha }}{\log \tau}$,such that with probability at least $1-\epsilon $,
\begin{align*}
    \frac{\left\Vert \boldsymbol{v}^{(M)} -\mathrm{d}_{\boldsymbol{\lambda}} f \right\Vert }{\left\Vert \partial _{\boldsymbol{x}}  f   \right\Vert } & \leq \frac{4 \sqrt{\kappa_{\bm \lambda}^2 + \kappa_{\bm x}^2 \frac{\beta^2}{\alpha^2}}}{\alpha} \left(1 + \frac{8 \sqrt{n}}{\delta}\tau^S\right) \sqrt{2 \log \frac{n (d_x + d_\lambda)}{\epsilon}} + \frac{2\beta}{\eta \alpha^2} \frac{8 \sqrt{n}}{\delta} \tau^S \\
    & \quad + \left\{ \frac{2\nu_g}{\alpha} \left(1 + \frac{\beta}{\alpha}\right) \left(1 + \frac{8 \sqrt{n}}{\delta} \tau^S\right) + \frac{\nu_f}{\|\partial_{\bm x} f\|} \frac{\beta}{\alpha} \right\} \left\| \hat{\bm{x}}(\bm{\lambda}) - \bm{x}(\bm{\lambda}) \right\| + \exp (-O(M)) .
\end{align*}
\end{theorem}

\begin{proof}
We first recall the inexact version of $\bm{v}^{(M)}$ obtained using $\hat{\bm{x}}(\bm{\lambda})$.
\begin{remark}[Explicit formula of $\bm{v}^{(M)}$ under inexact $\hat{\bm{x}}(\bm{\lambda})$] \label{rem:vm_inexact}
    $\bm{v}^{(M)}$ obtained \cref{alg:hgp} using the inexact solution $\hat{\bm{x}}(\bm{\lambda})$ is expressed as
    \begin{align}
        \bm{v}^{(M)} = - \eta \sum_{m=0}^{M-1} \partial_{\bm{x \lambda}}^2 \check{g}^{(m)} \hat{\boldsymbol{\Theta}} \prod_{m'=0}^{m-1} \left( \left( I - \eta \partial_{\bm x}^2 \check{g}^{(m')} \right) \hat{\boldsymbol{\Theta}} \right) \partial_{\bm x} f(\hat{\bm{x}}(\bm{\lambda}), \bm{\lambda}) + \partial_{\bm \lambda} f(\hat{\bm{x}}(\bm{\lambda}), \bm{\lambda}) ,
    \end{align}
    where
    \begin{align*}
        \partial_{\bm{x \lambda}}^2 \check{g}^{(m)} &= \mathrm{diag}\left(\left[\partial _{\boldsymbol{x}_i\boldsymbol{\lambda}_i} ^2 g_i(\hat{\boldsymbol{x}}_i(\boldsymbol{\lambda}), \boldsymbol{\lambda}_i; \hat{\xi}_i^{(m)})\right]_{i=1}^n\right) , \\
        \partial_{\bm{x}}^2 \check{g}^{(m)} &= \mathrm{diag}\left(\left[\partial _{\boldsymbol{x}_i} ^2 g_i(\hat{\boldsymbol{x}}_i(\boldsymbol{\lambda}), \boldsymbol{\lambda}_i; \hat{\xi}_i^{(m)})\right]_{i=1}^n\right) .
    \end{align*}
\end{remark}

From \cref{rem:vm_inexact}, we have
\begin{align*}
    & \bm{v}^{(M)} - \mathrm{d}_{\boldsymbol{\lambda}} f \\
    & = - \underbrace{\eta \sum_{m=0}^{M-1} \left\{ \partial_{\bm{x \lambda}}^2 \check{g}^{(m)} \hat{\boldsymbol{\Theta}} \prod_{m'=0}^{m-1} \left( \left( I - \eta \partial_{\bm x}^2 \check{g}^{(m')} \right) \hat{\boldsymbol{\Theta}} \right) \partial_{\bm x} f(\hat{\bm{x}}(\bm{\lambda}), \bm{\lambda}) -  \partial _{\boldsymbol{x}\boldsymbol{\lambda}}^{2} g   \boldsymbol{\Theta} \left(\left( \boldsymbol{I}-\eta \partial _{\boldsymbol{x}}^{2} g   \right) \boldsymbol{\Theta} \right)^{m} \partial _{\boldsymbol{x}} f(\hat{\bm{x}}(\bm{\lambda}), \bm{\lambda}) \right\}}_{\text{(A)}} \\
    & \quad +\underbrace{ \eta \partial _{\boldsymbol{x}\boldsymbol{\lambda}}^{2} g   \boldsymbol{\Theta} \sum\limits _{m=M}^{\infty }\left(\left( \boldsymbol{I}-\eta \partial _{\boldsymbol{x}}^{2} g   \right) \boldsymbol{\Theta} \right)^{m} \partial _{\boldsymbol{x}} f(\hat{\bm{x}}(\bm{\lambda}), \bm{\lambda}) }_{\text{(B)}} + \underbrace{\partial_{\bm \lambda} f(\hat{\bm{x}}(\bm{\lambda}), \bm{\lambda}) - \partial_{\bm \lambda} f(\bm{x}(\bm{\lambda}), \bm{\lambda})}_{\text{(C)}} .
\end{align*}

By the same arguments as in the proof of \cref{thm:error}, we have
\begin{align*}
    \| \text{(B)} \| = \exp\left(- O(M)\right).
\end{align*}
We can bound the term (C) using \cref{ass:glip}:
\begin{align*}
    \| \text{(C)} \| \le \nu_f \left\| \hat{\bm{x}}(\bm{\lambda}) - \bm{x}(\bm{\lambda}) \right\| .
\end{align*}
To bound the term (A), we follow the proof of \cref{thm:error} and obtain
\begin{align*}
    \text{(A)} & = \text{(A1)} + \text{(A2)} + \text{(A3)} + \text{(A4)} \\
    & \quad + \underbrace{\eta \sum_{m=0}^{M-1} \left(\partial_{\bm{x \lambda}}^2 \check{g}^{(m)} - \partial_{\bm{x \lambda}}^2 \hat{g}^{(m)}\right) \hat{\boldsymbol{\Theta}} \prod_{m'=0}^{m-1} \left( \left( I - \eta \partial_{\bm x}^2 \check{g}^{(m')} \right) \hat{\boldsymbol{\Theta}} \right) \partial _{\boldsymbol{x}} f}_{\text{(D1)}} \\
    & \quad - \underbrace{\eta \sum_{m=1}^{M-1} \partial_{\bm{x \lambda}}^2 g \boldsymbol{\Theta} \sum_{t=0}^{m-1} \left(\left( \boldsymbol{I}-\eta \partial _{\boldsymbol{x}}^{2} g   \right) \boldsymbol{\Theta} \right)^{t} \eta \left(\partial_{\bm{x}}^2 \check{g}^{(m-t-1)} - \partial_{\bm{x}}^2 \hat{g}^{(m)}\right) \hat{\boldsymbol{\Theta}} \prod_{m'=0}^{m-t-2} \left( \left( I - \eta \partial_{\bm x}^2 \check{g}^{(m')} \right) \hat{\boldsymbol{\Theta}} \right) \partial _{\boldsymbol{x}} f}_{\text{(D2)}} \\
    & \quad + \underbrace{\eta \sum_{m=0}^{M-1} \partial_{\bm{x \lambda}}^2 g \boldsymbol{\Theta} \prod_{m'=0}^{m-1} \left( \left( I - \eta \partial_{\bm x}^2 \hat{g}^{(m')} \right) \hat{\boldsymbol{\Theta}} \right) \left(\partial_{\bm x} f(\hat{\bm{x}}(\bm{\lambda}), \bm{\lambda}) - \partial _{\boldsymbol{x}} f\right)}_{\text{(D3)}} .
\end{align*}
We bound each term using \cref{ass:glip} below.

\paragraph{Bounding $\|\text{(D1)}\|$}
\begin{align}
    \|\text{(D1)} \| & \le \eta \sum_{m=0}^{M-1} \nu_g \left\| \hat{\bm{x}}(\bm{\lambda}) - \bm{x}(\bm{\lambda}) \right\| (1 - \eta \alpha)^m \left(1 + \frac{8 \sqrt{n}}{\delta} \tau^S\right)^{m+1} \|\partial_{\bm x} f\| \nonumber \\
    & = \|\partial_{\bm x} f\| \eta \nu_g \left\| \hat{\bm{x}}(\bm{\lambda}) - \bm{x}(\bm{\lambda}) \right\| \sum_{m=0}^{M-1} (1 - \eta \alpha)^{m} \left(1 + \frac{8 \sqrt{n}}{\delta} \tau^S\right)^{m+1}  \nonumber \\
    & = \|\partial_{\bm x} f\| \eta \nu_g \left\| \hat{\bm{x}}(\bm{\lambda}) - \bm{x}(\bm{\lambda}) \right\| \frac{1}{1 - (1 - \eta \alpha) \left(1 + \frac{8 \sqrt{n}}{\delta} \tau^S\right)} \left(1 + \frac{8 \sqrt{n}}{\delta} \tau^S\right) + \exp(- O(M)) \nonumber \\
    & \le \|\partial_{\bm x} f\| \frac{2}{\alpha} \nu_g \left\| \hat{\bm{x}}(\bm{\lambda}) - \bm{x}(\bm{\lambda}) \right\| \left(1 + \frac{8 \sqrt{n}}{\delta} \tau^S\right)  + \exp(- O(M)) . \label{eq:d1}
\end{align}

\paragraph{Bounding $\|\text{(D2)}\|$}
\begin{align}
    \|\text{(D2)}\| & = \eta \sum_{m=1}^{M-1} \beta \sum_{t=0}^{m-1} \eta \nu_g \left\| \hat{\bm{x}}(\bm{\lambda}) - \bm{x}(\bm{\lambda}) \right\| (1 - \eta \alpha)^{m-1} \left(1 + \frac{8 \sqrt{n}}{\delta} \tau^S\right)^{m-t} \|\partial_{\bm x} f\| \nonumber \\
    & = \|\partial_{\bm x} f\| \eta^2 \beta \nu_g \left\| \hat{\bm{x}}(\bm{\lambda}) - \bm{x}(\bm{\lambda}) \right\| \sum_{m=1}^{M-1} (1 - \eta \alpha)^{m-1} \sum_{t=0}^{m-1} \left(1 + \frac{8 \sqrt{n}}{\delta} \tau^S\right)^{m-t}  \nonumber \\
    & = \|\partial_{\bm x} f\| \eta^2 \beta \nu_g \left\| \hat{\bm{x}}(\bm{\lambda}) - \bm{x}(\bm{\lambda}) \right\| \sum_{m=1}^{M-1} (1 - \eta \alpha)^{m-1} \frac{\left(1 + \frac{8 \sqrt{n}}{\delta} \tau^S\right)^m - 1}{\frac{8 \sqrt{n}}{\delta} \tau^S} \left(1 + \frac{8 \sqrt{n}}{\delta} \tau^S\right)  \nonumber \\
    & = \|\partial_{\bm x} f\| \frac{\eta \beta}{\alpha} \nu_g \left\| \hat{\bm{x}}(\bm{\lambda}) - \bm{x}(\bm{\lambda}) \right\| \frac{\frac{8 \sqrt{n}}{\delta} \tau^S}{1 - (1 - \eta \alpha) \left(1 + \frac{8 \sqrt{n}}{\delta} \tau^S\right)} \frac{1 + \frac{8 \sqrt{n}}{\delta} \tau^S}{\frac{8 \sqrt{n}}{\delta} \tau^S} + \exp(- O(M)) \nonumber \\
    & \le \|\partial_{\bm x} f\| \frac{2\beta}{\alpha^2} \nu_g \left\| \hat{\bm{x}}(\bm{\lambda}) - \bm{x}(\bm{\lambda}) \right\| \left(1 + \frac{8 \sqrt{n}}{\delta} \tau^S\right) + \exp(- O(M)) . \label{eq:d2}
\end{align}

\paragraph{Bounding $\|\text{(D3)}\|$}
\begin{align}
    \|\text{(D3)}\| & \le \eta \sum_{m=0}^{M-1} \beta (1 - \eta \alpha)^m \left(1 + \frac{8 \sqrt{n}}{\delta} \tau^S\right)^m \nu_f \left\|\hat{\bm{x}}(\bm{\lambda}) - \bm{x}(\bm{\lambda})\right\| \nonumber \\
    & = \eta \beta \nu_f \left\|\hat{\bm{x}}(\bm{\lambda}) - \bm{x}(\bm{\lambda})\right\| \frac{1}{1 - (1 - \eta \alpha)\left(1 + \frac{8 \sqrt{n}}{\delta} \tau^S\right)} + \exp(-O(M)) \nonumber \\
    & \le \frac{\beta}{\alpha} \nu_f \left\|\hat{\bm{x}}(\bm{\lambda}) - \bm{x}(\bm{\lambda})\right\| + \exp(-O(M)) . \label{eq:d3}
\end{align}

Finally, \cref{thm:error2} follows from \cref{eq:a24}, \cref{eq:a13}, \cref{eq:d1}, \cref{eq:d2}, and \cref{eq:d3}.

\end{proof}

\end{document}